\documentclass{article}

\usepackage{microtype,tikz,bm}
\usepackage{graphicx}
\usepackage{subfigure}
\usepackage{amsmath}
\usepackage{amsfonts}
\usepackage{amssymb}
\usepackage{booktabs,epstopdf} 
\usepackage{amsthm}

\PassOptionsToPackage{numbers, compress}{natbib}
\RequirePackage{algorithm}
\RequirePackage{algorithmic}
\usepackage{wrapfig}
\usepackage{lipsum}
\usepackage[preprint]{neurips_2019}

\usepackage[utf8]{inputenc} 
\usepackage[T1]{fontenc}    
\usepackage{hyperref}       
\usepackage{url}            
\usepackage{booktabs}       
\usepackage{amsfonts}       
\usepackage{nicefrac}       
\usepackage{microtype}      
\usepackage{amssymb}

\usepackage{mathtools}
\DeclarePairedDelimiter\ceil{\lceil}{\rceil}

\renewcommand{\H}{\boldsymbol{H}}

\newcommand{\A}{\boldsymbol{A}}

\newcommand{\x}{\boldsymbol{x}}

\newtheorem{Lemma}{Lemma}

\newtheorem{Theorem}{Theorem}
\newtheorem{Def}{Definition}
\newtheorem{Corollary}{Corollary}

\newcommand{\cone}[1]{\textsf{cone}\left\{#1\right\}}

\newcommand{\T}{{\!\top\!}}
\DeclareMathOperator*{\minimize}{\rm minimize}

\newcommand\bd{\ensuremath{{\rm bd}}}

\usepackage{xcolor}
\usepackage{psfrag,framed}
\usepackage{lipsum}
\usepackage{chapterbib}
\PassOptionsToPackage{normalem}{ulem}
\usepackage{ulem}
\definecolor{shadecolor}{RGB}{220,220,220}

\definecolor{orange}{RGB}{255,107,0}

\title{Crowdsourcing via Pairwise Co-occurrences:
	Identifiability and Algorithms}

\author{%
  Shahana Ibrahim\\
School of Elec. Eng. \& Computer Sci.\\
  Oregon State University\\
  Corvallis, OR 97331 \\
  \texttt{ibrahish@oregonstate.edu} \\
   \And
  Xiao Fu\thanks{The work is supported in part by the National Science Foundation under projects ECCS 1808159 and NSF ECCS 1608961, and by the Army Research Office (ARO) under projects ARO W911NF-19-1-0247 and ARO W911NF-19-1-0407.} \\
School of Elec. Eng. \& Computer Sci.\\
Oregon State University\\
Corvallis, OR 97331 \\
\texttt{xiao.fu@oregonstate.edu} \\
   \And
Nikos Kargas\\
Department of Elec. \& Computer Eng.\\
University of Minnesota\\
Minneapolis, MN 55455 \\
\texttt{kaga005@umn.edu} \\
   \And
Kejun Huang\\
Department of Computing \& Info. Sci. \& Eng.\\
University of Florida\\
Gainesville, FL 32611 \\
\texttt{kejun.huang@ufl.edu} 
}

\begin{document}

\maketitle

\begin{abstract}
The data deluge comes with high demands for data labeling. Crowdsourcing (or, more generally, ensemble learning) techniques aim to produce accurate labels via integrating noisy, non-expert labeling from annotators. The classic Dawid-Skene estimator and its accompanying expectation maximization (EM) algorithm have been widely used, but the theoretical properties are not fully understood. Tensor methods were proposed to guarantee identification of the Dawid-Skene model, but the sample complexity is a hurdle for applying such approaches---since the tensor methods hinge on the availability of third-order statistics that are hard to reliably estimate given limited data. In this paper, we propose a framework using pairwise co-occurrences of the annotator responses, which naturally admits lower sample complexity. We show that the approach can identify the Dawid-Skene model under realistic conditions. We propose an algebraic algorithm reminiscent of convex geometry-based structured matrix factorization to solve the model identification problem efficiently, and an identifiability-enhanced algorithm for handling more challenging and critical scenarios. Experiments show that the proposed algorithms outperform the state-of-art algorithms under a variety of scenarios.
\end{abstract}

\section{Introduction}
\label{introduction}

\noindent
{\bf Background.}
The drastically increasing availability of data has successfully enabled many timely applications in machine learning and artificial intelligence.
At the same time, most supervised learning tasks, e.g., the core tasks in computer vision, natural language processing, and speech processing, heavily rely on labeled data.
However, labeling data is not a trivial task---it requires educated and knowledgeable annotators (which could be human workers or machine classifiers), to work under a reliable way. More importantly, it needs an effective mechanism to integrate the possibly different labeling from multiple annotators.
Techniques addressing this problem in machine learning are called {\it crowdsourcing} \cite{kittur2008crowdsourcing} or more generally, {\it ensemble learning} \cite{dietterich2000ensemble}.

Crowdsourcing has a long history in machine learning, which can be traced back to the 1970s \cite{dawid1979maximum}. Many models and methods have appeared since then \cite{karger2013efficient,karger2014budget,karger2011budget,snow2008cheap,welinder2010multidimensional,liu2012variational, traganitis2018blind}. Intuitively, if a number of reliable annotators label the same data samples, then {\it majority voting} among the annotators is expected to work well.
However, in practice, not all the annotators are equally reliable---e.g., different annotators could be specialized for recognizing different classes.
In addition, not all the annotators are labeling all the data samples, since data samples are often dispatched to different groups of annotators in a certain way. Under such circumstances, majority voting is not very promising. 

A more sophisticate way is to treat the crowdsourcing problem as a model identification problem. The arguably most popular generative model in crowdsourcing is the Dawid-Skene model \cite{dawid1979maximum}, where every annotator is assigned with a `confusion matrix' that decides the probability of an annotator giving class label $\ell$ when
the ground-truth label is $g$. If such confusion matrices and the probability mass function (PMF) of the ground-truth label can be identified, then a maximum likelihood (ML) or a maximum \textit{a posteriori} (MAP) estimator for the true label of any given sample can be constructed. The Dawid-Skene model is quite simple and succinct, and some of the model assumptions (e.g., the conditional independence of the annotator responses) are actually debatable. Nonetheless, this model has been proven very useful in practice \cite{raykar2010learning,traganitis2018blind,ghosh2011moderates,karger2014budget,liu2012variational,zhang2014spectral}.

Theoretical aspects for the Dawid-Skene model, however, are less well understood. In particular, it had been unclear if the model could be identified via the accompanying \textit{expectation maximization} (EM) algorithm proposed in the same paper \cite{dawid1979maximum}, until some recent works addressing certain special cases \cite{karger2014budget}. The works in \cite{traganitis2018blind,zhang2014spectral} put forth tensor methods for learning the Dawid-Skene model. These methods admit model identifiability, and also can be used to effectively initialize the classic EM algorithm provably \cite{zhang2014spectral}. 
The challenge is that tensor methods utilize third-order statistics of the data samples, which are rather hard to estimate reliably in practice given limited data \cite{huang2018learning}.

{\bf Contributions.}
In this work, we propose an alternative for identifying the Dawid-Skene model, without using third-order statistics.
Our approach is based on utilizing the pairwise co-occurrences of annotators' responses to data samples---which are second-order statistics and thus are naturally much easier to estimate compared to the third-order ones. We show that, by judiciously combining the co-occurrences between different annotator pairs, the confusion matrices and the ground-truth label's prior PMF can be provably identified, under realistic conditions (e.g., when there exists a relatively well-trained annotator among all annotators). 
This is reminiscent of nonnegative matrix theory and convex geometry \cite{fu2018nonnegative,gillis2014and}.
Our approach is also naturally robust to spammers as well as scenarios where every annotator only labels partial data.
We offer two algorithms under the same framework. The first algorithm is algebraic, and thus is efficient and suitable for handling very large-scale crowdsourcing problems. 
The second algorithm offers enhanced identifiability guarantees, and is able to deal with more critical cases (e.g., when no highly reliable annotators exist), with the price of using a computationally more involved iterative optimization algorithm.
Experiments show that both approaches outperform a number of competitive baselines.

\vspace{-.35cm}
\section{Background}
\vspace{-.25cm}
{\bf The Dawid-Skene Model.}
Let us consider a dataset $
\{{\bm f}_n \}_{n=1}^{N}$, where $\bm f_n \in\mathbb{R}^d$ is a data sample (or, feature vector) and $N$ is the number of samples. 
Each $\bm f_n$ belongs to one of $K$ classes. Let $y_n$ be the ground-truth label of the data sample ${\bm f}_n$. 
Suppose that there are $M$ annotators who work on the dataset $\{{\bm f}_n\}_{n=1}^N$ and provide labels. Let $X_{m}({\bm f}_{n})$ represent the response of the annotator $m$ to  ${\bm f}_n$. Hence, $X_m$ can be understood as a discrete random variable whose alphabet is $\{1,\ldots,K\}$. In crowdsourcing or ensemble learning, our goal is to estimate the true label corresponding to each item ${\bm f}_n$ from the $M$ annotator responses. Note that in a realistic scenario, an annotator will likely to only work on part of the dataset, since having all annotators work on all the samples is much more costly.


In 1979, Dawid and Skene proposed an intuitively pleasing model for estimating the `true response' of the patients from recorded answers \cite{dawid1979maximum},
which is essentially a crowdsourcing/ensemble learning problem. This model has sparked a lot of interest in the machine learning community \cite{raykar2010learning,traganitis2018blind,ghosh2011moderates,karger2014budget,liu2012variational,zhang2014spectral}.
The Dawid-Skene model in essence is a \textit{naive Bayesian model} \cite{robert2014machine}. In this model, the ground-truth label of a data sample is a latent discrete random variable, $Y$, whose values are different class indices. The ambient variables are the responses given by different annotators, denoted as $X_1,\ldots,X_M$, where $M$ is the number of annotators. The key assumption in the Dawid-Skene model is that given the ground-truth label, the responses of the annotators are conditionally independent. Of course, the Dawid-Skene model is a simplified version of reality, but has been proven very useful---and it has been a workhorse for crowdsourcing since its proposal.

Under the Dawid-Skene model, one can see that
\begin{equation}\label{eq:signalmodel}
\begin{aligned}
&{\sf Pr}(X_1 = k_1 , \dots, X_M = k_M) =\sum_{k=1}^K \prod_{m=1}^{M} {\sf Pr}(X_m = k_m | Y=k){\sf Pr}(Y=k),
\end{aligned}
\end{equation}
where $k\in\{1,\ldots,K\}$ denotes the index of a given class, and $k_m$ denotes the response of the $m$-th annotator. 
If one defines a series of matrices $\bm A_m\in\mathbb{R}^{K\times K}$ and let
\begin{equation}\label{eq:confusion}
\bm A(k_m,k):={\sf Pr}(X_m = k_m | Y=k),
\end{equation}
then  $\bm A_m\in\mathbb{R}^{K\times K}$ can be understood as the `confusion matrix' of annotator $m$: It contains all the conditional probabilities of annotator $m$ labeling a given data sample as from class $k_m$ while the ground-truth label is $k$. Also define a vector $\bm d\in\mathbb{R}^K$ such that
$\bm d(k):={\sf Pr}(Y=k);$
i.e., the prior PMF of the ground-truth label $Y$. Then the crowdsourcing problem boils down to estimating $\bm A_m$ for $m=1,\ldots,M$ and $\bm d$.

{\bf Prior Art.}
In the seminal paper \cite{dawid1979maximum}, Dawid and Skene proposed an EM-based algorithm to estimate ${\sf Pr}(X_m=k_m|Y=k)$ and ${\sf Pr}(Y=k)$.
Their formulation is well-motivated from an ML viewpoint, but also has some challenges.
First, it is unknown if the model is identifiable, especially when there is a large number of unrecorded responses (i.e., missing values)---but model identification plays an essential role in such estimation problems \cite{fu2018nonnegative}.
Second, since the ML estimator is a nonconvex optimization criterion, the solution quality of the EM algorithm is not easy to characterize in general. 
More recently, tensor methods were proposed to identify the Dawid-Skene model \cite{zhang2014spectral,traganitis2018blind}.
Take the most recent work in \cite{traganitis2018blind} as an example. The approach considers estimating the joint probability ${\sf Pr}(X_i=k_i,X_j=k_j,X_\ell=k_\ell)$ for different triples $i,j,\ell $. Such joint PMFs can be regarded as third-order tensors, and the confusion matrices and the prior $\bm d$ are latent factors of these tensors. 
The upshot is that identifiability of $\bm A_m$ and $\bm d$ can be elegantly established leveraging tensor algebra  \cite{sidiropoulos2017tensor,kolda2009tensor}. 
The challenge, however, is that reliably estimating ${\sf Pr}(X_i=k_i,X_j=k_j,X_\ell=k_\ell)$ is quite hard, since it normally needs a large number of annotator responses.
Another tensor method in \cite{zhang2014spectral} judiciously partitions the data and works with group statistics between three groups, which is reminiscent of the graph statistics proposed in \cite{anandkumar2014tensor}. The method is computationally more tractable, leveraging orthogonal tensor decomposition.  Nevertheless, the challenge again lies in sample complexity: the group/graph statistics are still third-order statistics.



\vspace{-.35cm}
\section{Proposed Approach}
\vspace{-.25cm}
In this section, we propose a model identification approach that only uses second-order statistics, in particular, pairwise co-occurrences ${\sf Pr}(X_i=k_i,X_j=k_j)$.

{\bf Problem Formulation.}
Let us consider the following pairwise joint PMF:
$$\text{Pr}(X_{m} =k_m,X_{\ell}=k_\ell) = \sum_{k=1}^{K}{\sf Pr}(Y=k) {\sf Pr}(X_m=k_m| Y=k) {\sf Pr}(X_\ell=k_\ell | Y=k).$$
Letting
$   \bm R_{m,\ell}(k_m,k_\ell) =    \text{Pr}(X_{m} =k_m,X_{\ell}=k_\ell),$
and using the matrix notations that we defined, we have
$\bm R_{m,\ell}(k_m,k_\ell)  = \sum_{k=1}^{K}{\sf Pr}(Y=k) {\sf Pr}(X_m=k_m| Y=k)  {\sf Pr}(X_\ell=k_\ell | Y=k) $---or, in a more compact form:
$$ \bm R_{m,\ell}(k_m,k_\ell) = \sum_{k=1}^K \bm d(k)\bm A_m(k_m,k)\bm A_\ell(k_\ell,k)~\Longleftrightarrow \bm R_{m,\ell}:=\bm A_m\bm D\bm A_\ell^\top,\label{eq:RADA}$$
where we have
$\bm D ={\rm Diag} (\bm d)$,
which is a diagonal matrix. 
Note that $\bm A_m$ is a confusion matrix, i.e., its columns are respectable probability measures. 
In addition, $\bm d$ is a prior PMF.
Hence, we have
\begin{equation}\label{eq:nn}
\bm 1^\top\bm A_m = \bm 1^\top,~\bm A_m\geq \bm 0,~\forall~m,\quad \bm 1^\top\bm d=1,~\bm d\geq \bm 0.
\end{equation}
%

In practice, $\bm R_{m,\ell}$'s are not available but can be estimated via sample averaging. Specifically,
if we are given the annotator responses $X_m ({\bm f}_n )$, then 
$$
\widehat{\bm R}_{m,\ell }(k_m ,k_\ell) =\frac{1}{|{\cal S}_{m,\ell}|} \sum_{n\in{\cal S}_{m,\ell}} I\left[X_{m}({\bm f}_n) = k_m , X_\ell(\bm f_n)= k_\ell\right],
$$
where ${\cal S}_{m,\ell}$ is the index set of samples which both annotators $m$ and $\ell$ have worked on.
Here, $I[\cdot]$ is an indicator function: If the event $E$ happens, then $I[E]=1$, and $I[E^c]=0$ otherwise. It is readily seen that
\begin{equation}\label{eq:expect}
\begin{aligned}
\mathbb{E}&\left[ I(X_{m}({\bm f}_n) = k_m , X_\ell(\bm f_n)= k_\ell ) \right] = \bm R_{m,\ell}(k_m,k_\ell),
\end{aligned}
\end{equation}
where the expectation is taken over data samples. 
Note that 
the sample complexity for reliably estimating $\bm R_{m,\ell}$ is much lower relative to that of estimating $\bm R_{m,n,\ell}$ \cite{zhang2014spectral,anandkumar2014tensor}, and the latter is needed in  tensor based methods, e.g., \cite{traganitis2018blind}. To be specific, to achieve $| \bm R_{m,\ell}(k_m,k_\ell)-\widehat{ \bm R}_{m,\ell}(k_m,k_\ell)|\leq \epsilon$ with a probability greater than $1-\delta$, ${\cal O}({\epsilon^{-2}(\rm log}\frac{1}{\delta}))$ joint responses from annotators $m$ and $\ell$ are needed. However, in order to attain the same accuracy for $\widehat{\bm R}_{m,n,\ell}(k_m,k_n,k_\ell)$, the number of joint responses from annotators $m$,$n$ and $\ell$ is required to be atleast ${\cal O}({K\epsilon^{-2}(\rm log}\frac{K}{\delta}))$, where $K$ is the number of classes (also see supplementary materials Sec. \ref{sample} for a short discussion).

{\bf An Algebraic Algorithm.}
Assume that we have obtained ${\bm R}_{m,\ell }$'s for different pairs of $m,\ell$. 
We now show how to identify $\bm A_m$'s and $\bm d$ from such second-order statistics.
Let us take the estimation of $\bm A_m$ as an illustrative example.
First, we construct a matrix ${\bm Z}_m$ as follows:
\begin{equation}
\bm Z_m = \begin{bmatrix} 
    \boldsymbol{R}_{m,m_1} ,\boldsymbol{R}_{m,m_2} ,\ldots,\boldsymbol{R}_{m,m_{T(m)}}
    \end{bmatrix},
\label{eqX}
\end{equation}
where $m_t\neq m$ for $t=1,\ldots,T(m)$ denote the indices of annotators who have co-labeled data samples with annotator $m$, and
the integer $T(m)$ denotes the number of such annotators.
Due to the underlying model of $\bm R_{m,\ell}$ in \eqref{eq:RADA}, we have
$$ Z_m = \left[\bm A_m\bm D\bm A_{m_1}^\top,\ldots,\bm A_m\bm D\bm A_{T(m)}^\top \right] =\bm A_m\left [\bm D\bm A_{m_1}^\top,\ldots,\bm D\bm A_{T(m)}^\top\right]\in\mathbb{R}^{K\times KT(m)}.$$
Let us define 
$ \bm H_m^\top = \left [\bm D\bm A_{m_1}^\top,\ldots,\bm D\bm A_{T(m)}^\top\right] \in\mathbb{R}^{ K\times  KT(m)}.$ This leads to the model
$   \bm Z_m =\A_m\H_m^\T. $
We propose to identify $\bm A_m$ from $\bm Z_m$. 
The key enabling postulate is that, among all annotators, some $\bm A_\ell$'s should be \textit{diagonally dominant}---if there exist annotators who are reasonably trained.
In other words, for a reasonable annotator $\ell$, ${\sf Pr}(X_\ell=j|Y=j)$ should be greater than ${\sf Pr}(X_\ell=j|Y=k)$ and ${\sf Pr}(X_\ell=j|Y=i)$ for $k,i\neq j$.
To see the intuition of the algorithm, consider an ideal case where for each class $k$, there exists an annotator $m_{t(k)}\in\{m_1,\ldots,m_{T(m)}\}$ such that 
\begin{equation}\label{eq:anchor}
{\sf Pr}(X_{m_{t(k)}}=k|Y=k)=1,\quad {\sf Pr}(X_{m_{t(k)}}=k|Y=j)=0,\quad j\neq k.
\end{equation}
This physically means that annotator $m_{t(k)}$ is very good at recognizing class $k$ and never confuses other classes with class $k$.
Under such circumstances, one can use the following procedure to identify $\bm A_m$.
First, let us normalize the columns of $\bm Z_m$ via $\overline{\bm Z}_m(:,q)={\bm Z}_m(:,q)/\|{\bm Z}_m(:,q)\|_1$ for $q=\{1,\ldots,KT(m)\}$.
This way, we have a normalized model $\overline{\bm Z}_m=\overline{\bm A}_m \overline{\bm H}_m^\top$, where
	\begin{align}
	  &\overline{\bm A}_m(:,k) =\frac{\bm A_m(:,k)}{\|\bm A_m(:,k)\|_1}=\bm A_m(:,k), \quad \overline{\bm H}_m(q,:) =     \frac{{\bm H}_m(q,:)\|\bm A_m(:,k)\|_1}{\|\bm Z_m(:,q)\|_1}.  
	\end{align}
where the second equality above is because $\|\bm A_m(:,k)\|_1=1$ [cf. Eq.~\eqref{eq:nn}].
After normalization, it can be verified that 
\begin{equation}\label{eq:stocha}
\overline{\bm H}_m\bm 1=\bm 1,~\overline{\bm H}_m\geq \bm 0,
\end{equation}
i.e., all the rows of $\overline{\bm H}_m$ reside in the $(K-1)$-probability simplex.
In addition, by the assumption in \eqref{eq:anchor}, it is readily seen that there exists $\varLambda_q =\{q_1,\ldots,q_K\}\subset \{1,\ldots,L_m\}$ where $L_m=KT(m)$ such that 
\begin{equation}\label{eq:sep}
\overline{\bm H}_m(\varLambda_q,:) = \bm I_K, 
\end{equation}
i.e., an identity matrix is a submatrix of $\overline{\bm H}_m$ (after proper row permutations). Consequently, we have ${\bm A}_m = \overline{\bm Z}_m(:,\varLambda_q)$---i.e., $\bm A_m$ can be identified from $\overline{\bm Z}_m$ up to column permutations.
The task also boils down to identifying $\varLambda_q$. This turns out to be a well-studied task in the context of {\it separable nonnegative matrix factorization} \cite{Gillis2012,gillis2014and,fu2018nonnegative}, and an algebraic algorithm exists:

\boxed{
\begin{minipage}{\linewidth}
	\begin{equation}\label{eq:spa}
		\widehat{q}_k = \arg\max_{q\in\{1,\ldots,L_m\}}~\left\|\bm P^\perp_{\widehat{\bm A}_m(:,1:k-1)}\overline{\bm Z}_m(:,q)\right\|_2^2,~\forall k.
	\end{equation}
	where $\widehat{\bm A}_m(:,1:k-1)=[\overline{\bm Z}_m(:,\widehat{q}_1),\ldots,\overline{\bm Z}_m(:,\widehat{q}_{k-1})]$ and $\bm P^\perp_{\widehat{\bm A}_m(:,1:k-1)}$ is a projector onto the orthogonal  complement of ${\rm range}(\widehat{\bm A}_m(:,1:k-1))$ and we let $\bm P^\perp_{\widehat{\bm A}_m(:,1:0)}:={\bm I}$.
\end{minipage}
}

It has been shown in \cite{Gillis2012,arora2012practical} that the so-called \textit{successive projection algorithm} (SPA) in Eq.~\eqref{eq:spa} identifies $\varLambda_q$ in $K$ steps. This is a very plausible result, since the procedure admits Gram-Schmitt-like lightweight steps and thus is quite scalable. See more details in Sec.~\ref{app:spa}.

Each of the $\bm A_m$'s can be estimated from the corresponding $\overline{\bm Z}_m$ by repeatedly applying SPA, and we call this simple procedure\textit{ multiple SPA} (\texttt{MultiSPA}) as we elaborate in Algorithm~\ref{algo:MultiSPA}.

	\begin{minipage}{\textwidth}

	\begin{algorithm}[H]
		\caption{\texttt{MultiSPA}}\label{algo:MultiSPA}
		\label{alg:example}
		\begin{algorithmic}
			\STATE {\bfseries Input:} Annotator Responses $\{X_m(\bm f_n)\}$.
			\STATE {\bfseries Output:} $\widehat{\bm A}_m$ for $m=1,\ldots,M$, $\widehat{\bm d}$.
			\STATE estimate second order statistics $\widehat{{\bm R}}_{m,\ell }$;
			\FOR{$m=1$ {\bfseries to} $M$}
			\STATE construct $\widehat{\bm Z}_m$ and normalize columns to unit $\ell_1$ norm;
			\STATE estimate $\widehat{\bm A}_m$ using Eq.~\eqref{eq:spa};
			\ENDFOR
			\STATE fix permutation mismatch between $\widehat{\bm A}_m$ and $\widehat{\bm A}_\ell$ for all $m\neq \ell$; 
			\STATE estimate $\widehat{\bm D}=\widehat{\bm A}_m^{-1}\bm R_{m,\ell}(\widehat{\bm A}_\ell^\top)^{-1}$ (and take average over all pairs $(m,\ell)$ if needed).;
			\STATE extract the prior $\widehat{\bm d}={\rm diag}(\widehat{\bm D})$.
		\end{algorithmic}

	\end{algorithm}

\end{minipage}



Of course, assuming that \eqref{eq:anchor} or \eqref{eq:sep} holds perfectly may be too ideal. 
It is more likely that there exist some annotators who are good at recognizing certain classes, but still have some possibilities of being confused. 
It is of interest to analyze how SPA can do under such conditions.
Another challenge is that one may not have $\bm R_{m,\ell}$ perfectly estimated, since only limited number of samples are available. It is desirable to understand the sample complexity of applying SPA to Dawid-Skene identification.
We answer these two key technical questions in the following theorem:

\begin{Theorem}\label{thm:spa}
	Assume that annotators $m$ and $t$ co-label at least $S$ samples $\forall t\in\{m_1,\ldots,m_{T(m)}\}$, and that $\widehat{\bm Z}_m$ is constructed using $\widehat{\bm R}_{m,m_{T(m)}}$'s according to Eq. \eqref{eqX}.
	Also assume that the constructed $\widehat{\bm Z}_m$ satisfies $\|\widehat{\bm Z}_m(:,l)\|_1 \ge \eta, \forall l \in \{1,\dots KT(m)\}$, where $\eta \in (0,1]$. Suppose that ${\rm rank}(\bm A_m)={\rm rank}(\bm D)=K$ for $m=1,\ldots,M$, and that for every class index $k\in\{1,\ldots,K\}$, there exists an annotator $m_{t(k)}\in\{m_1,\ldots,m_{T(m)}\}$ such that
	\begin{equation}\label{eq:thm1}
	\begin{aligned}
	{\sf Pr}&(X_{m_{t(k)}}=k|Y=k) \geq (1-\epsilon)\sum_{j=1}^K {\sf Pr}(X_{m_{t(k)}}=k|Y=j), 
	\end{aligned}
	\end{equation}
	where $\epsilon\in [0,1]$. Then, if $\epsilon \le \mathcal{O}\left(\max\left(K^{-1}\kappa^{-3}(\bm{A}_m),\sqrt{{\rm ln}(1/\delta)}(\sigma_{\rm max}(\bm{A}_m)\sqrt{S}\eta)^{-1}\right)\right)$, with probability greater than $1-\delta$, the SPA algorithm in \eqref{eq:spa} can estimate an $\widehat{\bm A}_m$ such that
	\begin{equation}\label{eq:errbound}
	 \left(\min_{\bm \Pi}\|\widehat{\bm A}_m\bm \Pi - {\bm A}_m\|_{2,\infty}\right)\leq \mathcal{O}\bigg(\sqrt{K}\kappa^2(\bm{A}_m) \max \left(\sigma_{\rm max}(\bm{A}_m)\epsilon,\sqrt{{\rm ln}(1/\delta)}(\sqrt{S}\eta)^{-1}\right)\bigg)
	\end{equation}      
where $\bm \Pi\in\mathbb{R}^{K\times K}$ is a permutation matrix, $\|\bm Y\|_{2,\infty}=\max_\ell\|\bm Y(:,\ell)\|_2$, $\sigma_{\rm max}(\bm{A}_m)$ is the largest singular value of $\bm A_m$, and $\kappa(\bm{A}_m)$ is the condition number of $\bm A_m$.
\end{Theorem}
 In the above Theorem, the assumption $\|\widehat{\bm Z}_m(:,l)\|_1 \ge \eta$ means that the proposed algorithm favors cases where more co-occurrences are observed, since $\widehat{\bm Z}_m$'s elements are averaged number of co-occurrences---which makes a lot of sense. In addition, Eq.~\eqref{eq:thm1} relaxes the ideal assumption in \eqref{eq:anchor}, allowing the `good annotator' $m_{t(k)}$ to confuse class $j\neq k$ with class $k$ up to a certain probability, thereby being more realistic.
The proof of Theorem~\ref{thm:spa} is reminiscent of the noise robustness of the SPA algorithm \cite{Gillis2012,arora2012practical}; see the supplementary materials (Sec.~\ref{app:spa}).
A direct corollary is as follows:
\begin{Corollary}\label{col:spa}
	Assume that the conditions in Theorem~\ref{thm:spa} hold for $\widehat{\bm Z}_m$ and $\A_m$, $\forall m\in\{1,\ldots,M\}$. Then, the estimation error bound in \eqref{eq:errbound} holds for every \texttt{MultiSPA}-output $\widehat{\bm A}_m$, $\forall m\in\{1,\ldots,M\}$.
\end{Corollary}

	Theorem~\ref{thm:spa} and Corollary~\ref{col:spa} are not entirely surprising due to the extensive research on SPA-like algorithms \cite{arora2012practical,Gillis2012,fu2014self,VCA,VMAX}. The implication for crowdsourcing, however, is quite intriguing. First, one can see that if an annotator $m$ does not label all the data samples, it does not necessarily hurt the model identifiability---as long as annotator $m$ has co-labeled some samples with a number of other annotators, identification of $\bm A_m$ is possible.
	Second, assume that there exists a well-trained annotator $m^\star$ whose confusion matrix is diagonally dominant, then for every annotator $m$ who has co-labeled samples with annotator $m^\star$, the matrix $\overline{\bm H}_m$ can easily satisfy \eqref{eq:thm1} by letting $m_{t(k)}=m^\star$ for all $k$. In practice, one would not know who is $m^\star$---otherwise the crowdsourcing problem would be trivial. However, one can design a dispatch strategy such that every pair of annotators $m$ and $\ell$ co-label a certain amount of data. This way, it guarantees that $\bm A_{m^\star}$ appears in everyone else's $\bm H_m$ and thus ensures identifiability of all $\bm A_m$'s for $m\neq m^\star$. 
	This insight may shed some light on how to effectively dispatch data to annotators.

Another interesting question to ask is \textit{does having more annotators help}? 
Intuitively, having more annotators should help: If one has more rows in $\overline{\bm H}_m$, then it is more likely that some rows approach the vertices of the probability simplex---which can then enable SPA. 
We use the following simplified generative model and theorem to formalize the intuition:
\begin{Theorem}\label{thm:Lm}
	Let $\rho > 0 ,\varepsilon >0$, and assume that the rows of $\overline{\bm H}_m$ are generated within the $(K-1)$-probability simplex uniformly at random. If the number of annotators satisfies $M \ge \mathcal{O}\left(\frac{\varepsilon^{-2(K-1)}}{K}{\rm log}\left(\frac{K}{\rho}\right)\right),$ then, with probability greater than or equal to $1-\rho$, there exist rows of $\overline{\bm H}_m$ indexed by $ q_1,\dots  q_K$ such that 
$\|\overline{\bm H}_m( q_k, :) - \bm e_k^\T \|_2 \le \varepsilon,~k=1,\ldots,K.$
\end{Theorem}
Note that Theorem~\ref{thm:Lm} implies \eqref{eq:thm1} under proper $\varepsilon$ and $\epsilon$---and thus having more annotators indeed helps identify the model.
The above can be shown by utilizing the Chernoff-Hoeffding inequality, and
the detailed proof can be found in the supplementary materials (Sec.~\ref{app:thm2}).

 After obtaining $\widehat{\A}_m$'s, $\bm d$ can be estimated via various ways---see the supplementary materials in Sec. \ref{prior}. Using $\widehat{\bm d}$ and $\widehat{\A}_m$'s together, ML and MAP estimators for the true labels can be built up \cite{traganitis2018blind}.


\section{ Identifiability-enhanced Algorithm}
The \texttt{MultiSPA} algorithm is intuitive and lightweight, and is effective as we will show in the experiments.
One concern is that perhaps the assumption in \eqref{eq:thm1} may be violated in some cases.
In this section, we propose another model identification algorithm that is potentially more robust to critical scenarios. Specifically, we consider the following feasibility problem:

\boxed{
	\begin{minipage}{\linewidth}
	\begin{subequations}\label{eq:coupled}
		\begin{align}
	{\rm find}&\quad \{\bm A_m\}_{m=1}^M,~{\bm D}\\
	{\rm subject~to}&\quad \bm R_{m,\ell}=\bm A_m\bm D\bm A_\ell^\top,~\forall m,\ell\in\{1,\ldots,M\} \label{eq:R}\\
	&\quad {\bm 1}^\top\bm A_m =\bm 1^\top,~\bm A_m\geq \bm 0,~\forall m,~\bm 1^\top\bm d=1,~\bm d\geq \bm 0. \label{eq:dsum}
		\end{align}
	\end{subequations}

	\end{minipage}
}

The criterion in \eqref{eq:coupled} seeks confusion matrices and a prior PMF that fit the available second-order statistics.
The constraints in \eqref{eq:dsum} reflect the fact that the columns of $\bm A_m$'s are conditional PMFs and the prior $\bm d$ is also a PMF.

To proceed, let us first introduce the following notion from convex geometry \cite{fu2018nonnegative,lin2014identifiability}:
\begin{Def} \label{defsuf}
	(Sufficiently Scattered) 	A nonnegative matrix $\bm H\in\mathbb{R}^{L\times K}$ is sufficiently scattered if 
   1) ${\sf cone} \{\bm H^\top\} \supseteq {\cal C}$, and 2) ${\sf cone} \{\bm H^\top\}^* \cap \bd{\cal C}^* = \{\lambda\bm e_k~|~\lambda\geq 0,k=1,...,K \}$.
	Here,
	$\mathcal{C} = \{\x | \x^\T\mathbf{1} \geq \sqrt{K-1}\|\x\|_2\}$,
	$\mathcal{C}^* = \{\x | \x^\T\mathbf{1} \geq\|\x\|_2\}$. In addition, ${\sf cone} \{\bm H^\T\}=\{{\bm x}|{\bm x}=\bm H^\T\bm \theta,~\forall \bm \theta\geq{\bm 0}\}$ and { ${\sf cone} \{\bm H^\T\}^\ast = \{{\bm y}|\x^\T \bm y \geq{ 0}, ~\forall \x\in{\sf cone} \{\H^\T\}\}$ are the conic hull of $\bm H^\T$ and its dual cone}, respectively,
	and ${\rm bd}$ is the boundary of a closed set.
\end{Def}
The sufficiently scattered condition has recently emerged in convex geometry-based matrix factorization \cite{lin2014identifiability,fu2018identifiability}.
This condition models how the rows of $\bm H$ are spread in the nonnegative orthant.
In principle, the sufficiently scattered condition is much easier to be satisfied relative to the condition as in \eqref{eq:sep}, or, the so-called \textit{separability condition} under the context of nonnegative matrix factorization \cite{donoho2003does,Gillis2012}.
$\bm H$ satisfying the separability condition is the extreme case, meaning that ${\sf cone}\{\bm H^\T\}=\mathbb{R}_+^K$.
However, the sufficiently scattered condition only requires ${\cal C}\subseteq {\sf cone}\{\bm H^\T\}$---which is naturally much more relaxed; also see \cite{fu2018nonnegative} and the supplementary materials for detailed illustrations (Sec.~\ref{app:convex_geometry}).

Regarding identifiability of $\bm A_1,\ldots,\bm A_M$ and $\bm d$, we have the following result:
\begin{Theorem}\label{thm:couple}
Assume that ${\rm rank}(\bm D)={\rm rank}(\bm A_m)=K$ for all $m=1,\ldots,M$, and that there exist two subsets of the annotators, indexed by ${\cal P}_1$ and ${\cal P}_2$, where ${\cal P}_1\cap {\cal P}_2=\emptyset$ and ${\cal P}_1\cup {\cal P}_2 \subseteq \{1,\ldots,M\}$. Suppose that from ${\cal P}_1$ and ${\cal P}_2$ the following two matrices can be constructed:
$\bm H^{(1)} =[\bm A_{m_1}^\top,\ldots,\bm A^\top_{m_{|{\cal P}_1|}}]^\top$,~$
\bm H^{(2)}=[\bm A_{\ell_1}^\top,\ldots,\bm A^\top_{\ell_{|{\cal P}_2|}}]^\top$,
where $m_t\in{\cal P}_1$ and $\ell_j\in{\cal P}_2$. Furthermore, assume that i) both $\bm H^{(1)}$ and $\bm H^{(2)}$ are \textit{sufficiently scattered}; ii) all $\bm R_{m_t,\ell_j}$'s for $m_t\in{\cal P}_1$ and $\ell_j\in{\cal P}_2$ are available; and iii) for every $m\notin{\cal P}_1\cup{\cal P}_2$ there exists a ${\bm R}_{m,r}$ available, where $r\in{\cal P}_1\cup{\cal P}_2$.
Then, solving Problem~\eqref{eq:coupled} recovers $\bm A_m$ for $m=1,\ldots,M$ and $\bm D={\rm Diag}(\bm d)$ up to identical column permutation.
\end{Theorem}

The proof of Theorem~\ref{thm:couple} is relegated to the supplementary results (Sec. \ref{app:thm3}).
Note that the theorem holds under the the existence of ${\cal P}_1$ and ${\cal P}_2$, but there is no need to know the sets \textit{a priori}.
Generally speaking, a `taller' matrix $\bm H^{(i)}$ would have a better chance to have its rows sufficiently spread in the nonnegative orthant under the same intuition of Theorem~\ref{thm:Lm}. Thus, having more annotators also helps to attain the sufficiently scattered condition. 
Nevertheless, formally showing the relationship between the number of annotators and $\H^{(i)}$ for $i=1,2$ being sufficiently scattered is more challenging than the case in Theorem~\ref{thm:Lm}, since the sufficiently scattered condition is a bit more abstract relative to the separability condition---the latter specifically assumes $\bm e_k$'s exist as rows of $\H^{(i)}$ while the former depends on the `shape' of the conic hull of $(\H^{(i)})^\T$, which contains an infinite number of cases.
Towards this end, let us first define the following notion:

\begin{Def} \label{def:suffepsilon}
		Assume that there exist $\widetilde{\bm H}\in\mathbb{R}^{L\times K}$ such that $\widetilde{\H}$ is sufficiently scattered. Also assume ${\cal V}$ is the row index set of $\widetilde{\H}$ such that $\widetilde{\H}({\cal V},:)$ collects the extreme rays of ${\sf cone}\{ {\widetilde{\H}}~^\T\}$.
		If there exist row indices $\ell_v \in \{1,\dots,L\}$ for all $v\in{\cal V}$, such that $\|\widetilde{\bm H}(v,:)-\H(\ell_v,:)\|_2\leq  \varepsilon$, then $\H\in\mathbb{R}^{L\times K}$ is called $\varepsilon$-sufficiently scattered.
\end{Def}
One can see that an $\varepsilon$-sufficiently scattered matrix is sufficiently scattered when $\varepsilon\rightarrow 0$. With this definition, we show the following theorem:
\begin{Theorem}\label{thm:SSM}
	Let $\rho>0, \frac{\alpha}{2} > \varepsilon > 0, $, and assume that the rows of ${\bm H}^{(1)}$ and ${\bm H}^{(2)}$  are generated from $\mathbb{R}^{K}$ uniformly at random. If the number of annotators satisfies $M \ge \mathcal{O}\left(\frac{(K-1)^2}{K\alpha^{2(K-2)}\varepsilon^{2}}{\rm log}\left(\frac{K(K-1)}{\rho}\right)\right)$, where $\alpha=1$ for $K=2$, $\alpha=2/3$ for $K=3$ and $\alpha=1/2$ for $K>3$,	
  then with probability greater than or equal to $1-\rho$, ${\bm H}^{(1)}$ and ${\bm H}^{(2)}$ are $\varepsilon$-sufficiently scattered. 
\end{Theorem}

 The proof of Theorem \ref{thm:SSM} is relegated to the supplementary materials (Sec.~\ref{app:thm4}). One can see that to satisfy $\varepsilon$-sufficiently scattered condition, $M$ is smaller than that in Theorem~\ref{thm:Lm}. 
Conditions \textit{i)-iii)} in Theorem~\ref{thm:couple} and Theorem \ref{thm:SSM}  together imply that if we have enough annotators, and if many pairs co-label a certain number of data, then it is quite possible that one can identify the Dawid-Skene model via simply finding a feasible solution to \eqref{eq:coupled}. This feasibility problem is nonconvex, but can be effectively approximated; see the supplementary materials (Sec. \ref{algo}). In a nutshell, we reformulate the problem as a Kullback-Leibler (KL) divergence-based constrained fitting problem and handle it using alternating optimization.
Since nonconvex optimization relies on initialization heavily, we use \texttt{MultiSPA} to initialize the fitting stage---which we will refer to as the \texttt{MultiSPA-KL} algorithm.

\vspace{-.35cm}
\section{Experiments}
\vspace{-.25cm}

{\bf Baselines.}
The performance of the proposed approach is compared with a number of competitive baselines, namely, \texttt{Spectral-D\&S} \cite{zhang2014spectral}, \texttt{TensorADMM} \cite{traganitis2018blind}, and \texttt{KOS} \cite{karger2013efficient}, \texttt{EigRatio} \cite{dalvi2013aggregating}, \texttt{GhoshSVD} \cite{ghosh2011moderates} and  \texttt{MinmaxEntropy} \cite{zhou2014aggregated}. 
The performance of the \texttt{Majority Voting} scheme and the Majority Voting initialized Dawid-Skene (\texttt{MV-D\&S}) estimator \cite{dawid1979maximum} are also presented. 
We also use \texttt{MultiSPA} to initialize EM algorithm (named as \texttt{MultiSPA-D\&S}).
 Note that \texttt{KOS}, \texttt{EigRatio} and \texttt{MinmaxEntropy} work with more complex models relative to the Dawid-Skene model, but are considered as good baselines for the crowdsourcing/ensemble learning tasks.
After identifying the model parameters, we construct a MAP predictor following \cite{traganitis2018blind} and observe the result. The algorithms are coded in Matlab.

\color{black}

\noindent
{\bf Synthetic-data Simulations.} Due to page limitations, synthetic data experiments demonstrating model identifiability of the proposed algorithms are presented in the supplementary materials (Sec.~\ref{app:exp}).

\noindent
{\bf Integrating Machine Classifiers.}
We employ different UCI datasets (\url{https://archive.ics.uci.edu/ml/datasets.html}; details in Sec.~\ref{real_dataset}).
For each of the datasets under test, we use a collection of different classification algorithms to annotate the data samples. 
Different classification algorithms from the MATLAB machine learning toolbox (\url{https://www.mathworks.com/products/statistics.html}) such as various $k$-nearest neighbour classifiers, support vector machine classifiers, and decision tree classifiers are employed to serve as our machine annotators. In order to train the annotators, we use $20\%$ of the samples to act as training data.
After the data samples are trained, we use the annotators to label the unseen data samples.
In practice, not all samples are labeled by an annotator due to several factors such as annotator capacity, difficulty of the task, economical issues and so on.
To simulate such a scenario, each of the trained algorithms is allowed to label a data sample with probability $p \in (0,1]$. We test the performance of all the algorithms under different $p$'s---and a smaller $p$ means a more challenging scenario. All the results are averaged from 10 random trials. 

Table~\ref{tab:real} shows the classification error of the algorithms under test. Since \texttt{GhoshSVD} and \texttt{EigenRatio} works only on binary tasks, they are not evaluated for the Nursery dataset where $K=4$. The `single best' and `single worst' rows correspond to the results of using the classifiers individually when $p=1$, as references. The best and second-best performing algorithms are highlighted in the table. One can see that the proposed methods are quite promising for this experiment.
Both algorithms largely outperform the tensor based methods \texttt{TensorADMM} and \texttt{Spectral-D\&S} in this case, perhaps because the limited number of available samples makes the third-order statistics hard to estimate. It is also observed that the proposed algorithms enjoy favorable runtime;s ee supplementary materials (cf. Table \ref{tab:realtime} in Sec. \ref{real_dataset}). Using the \texttt{MultiSPA} to initialize EM (i.e. \texttt{MultiSPA-D\&S}) also works well, which offers another viable option that strikes a good balance between runtime and accuracy.

\begin{table}[t!]
  \centering
  \caption{Classification Error ($\%$) on UCI Datasets; see runtime tabulated in Sec. \ref{real_dataset}.}
  \resizebox{\linewidth}{!}{\large
    \begin{tabular}{|l|l|l|l|l|l|l|l|l|l|}
    \toprule
          & \multicolumn{3}{c|}{\textbf{Nursery}} & \multicolumn{3}{c|}{\textbf{Mushroom}} & \multicolumn{3}{c|}{\textbf{Adult}} \\
    \midrule
    \textbf{Algorithms} & \multicolumn{1}{l|}{\textbf{$p=1$}} & \multicolumn{1}{l|}{\textbf{$p=0.5$}} & \multicolumn{1}{l|}{\textbf{$p=0.2$}} & \multicolumn{1}{l|}{\textbf{$p=1$}} & \multicolumn{1}{l|}{\textbf{$p=0.5$}} & \multicolumn{1}{l|}{\textbf{$p=0.2$}} & \multicolumn{1}{l|}{\textbf{$p=1$}} & \multicolumn{1}{l|}{\textbf{$p=0.5$}} & \multicolumn{1}{l|}{\textbf{$p=0.2$}} \\
    \midrule
    \texttt{MultiSPA} & 2.83  & 4.54  & 17.96 & 0.02 & 0.293 & 6.35  & \textbf{15.71} & \textbf{16.05} & 17.66 \\
    \midrule
    \texttt{MultiSPA-KL} & \textbf{2.72} & \textbf{4.26} & \textbf{13.06} & \textbf{0.00} & \textbf{0.152} & \textbf{5.89} & \textbf{15.66} & \textbf{15.98} & \textbf{17.63} \\
    \midrule
    \texttt{MultiSPA-D\&S} & \textbf{2.82} & \textbf{4.44} & 13.39 & \textbf{0.00} & 0.194 & 6.17  & 15.74 & 16.29 & 23.88 \\
    \midrule
    \texttt{Spectral-D\&S} & 3.14  & 37.2  & 44.29 & \textbf{0.00} & 0.198 & 6.17  & 15.72 & 16.31 & 23.97 \\
    \midrule
    \texttt{TensorADMM} & 17.97 & 7.26  & 19.78 & 0.06 & 0.237 & 6.18  & 15.72 & \textbf{16.05} & 25.08 \\
    \midrule
    \texttt{MV-D\&S} & 2.92  & 66.48 & 66.61 & \textbf{0.00} & 47.99 & 48.63 & 15.76 & 75.21 & 75.13 \\
    \midrule
    \texttt{Minmax-entropy}    & 3.63  & 26.31 & \textbf{11.09} & \textbf{0.00} & \textbf{0.163} & 8.14  & 16.11 & 16.92 & \textbf{15.64} \\
    \midrule
    \texttt{EigenRatio} &  N/A    &  N/A     &    N/A   & 0.06 & 0.329 & \textbf{5.97} & 15.84 & 16.28 & 17.69 \\
    \midrule
    \texttt{KOS}   & 4.21  & 6.07  & 13.48 & 0.06 & 0.576 & 6.42  & 17.19 & 24.97 & 38.29 \\
    \midrule
    \texttt{Ghosh-SVD} &   N/A    &   N/A     &  N/A    & 0.06 & 0.329 & \textbf{5.97} & 15.84 & 16.28 & 17.71 \\
    \midrule
    \texttt{Majority Voting} & 2.94  & 4.83  & 19.75 & 0.14 & 0.566 & 6.57  & 15.75 & 16.21 & 20.57 \\
    \midrule
    Single Best & 3.94  &    N/A   &    N/A   & 0.00     &   N/A     &   N/A    & 16.23 &  N/A      &  N/A \\
    \midrule
    Single Worst & 15.65 &    N/A   &   N/A    & 7.22  &    N/A    &   N/A   & 19.27 & N/A       &  N/A \\
    \bottomrule
    \end{tabular}%
    }
  \label{tab:real}%
    \vspace{-.35cm}
\end{table}%

\noindent
{\bf Amazon Mechanical Turk Crowdsourcing Data.} 
In this section, the performance of the proposed algorithms are evaluated using the Amazon Mechanical Turk (AMT) data (\url{https://www.mturk.com}) in which human annotators label various classification tasks. Data description is given in the supplementary materials Sec. \ref{real_dataset}. 
Table \ref{tab:amazon} shows the classification error and the runtime performance of the algorithms under test. One can see that \texttt{MultiSPA} has a very favorable execution time, because it is a Gram-Schmitt-like algorithm. \texttt{MultiSPA-KL} uses more time, because it is an iterative optimization method---with better accuracy paid off. Since $\texttt{TensorADMM}$ algorithm does not scale well, the results are not reported for very large datasets (i.e., TREC and RTE). 
Similar as before, since Web and Dog are multi-class datasets, \texttt{EigenRatio} and $\texttt{GhoshSVD}$ are not applicable.
From the results, it can be seen that the proposed algorithms outperform many existing crowdsourcing algorithms in both classification accuracy and runtime. In particular, one can see that the algebraic algorithm $\texttt{MultiSPA}$ gives very similar results compared to the computationally much more involved algorithms. This shows the potential for its application in big data crowdsourcing.

\begin{table}[t!]
  \centering
   \caption{Classification Error ($\%$) and Run-time (sec) : AMT Datasets}
   \resizebox{\linewidth}{!}{\Huge
    \begin{tabular}{|l|l|l|l|l|l|l|l|l|l|l|}
    \toprule
    \textbf{Algorithms       } & \multicolumn{2}{c|}{\textbf{TREC          }} & \multicolumn{2}{c|}{\textbf{Bluebird       }} & \multicolumn{2}{c|}{\textbf{RTE             }} & \multicolumn{2}{c|}{\textbf{Web            }} & \multicolumn{2}{c|}{\textbf{Dog             }} \\
    \midrule
          & \textbf{(\%) Error} & \textbf{(sec) Time} & \textbf{(\%) Error} & \textbf{(sec) Time} & \textbf{(\%) Error} & \textbf{(sec) Time} & \textbf{(\%) Error} & \textbf{(sec) Time} & \textbf{(\%) Error} & \multicolumn{1}{l|}{\textbf{(sec) Time}} \\
    \midrule
    \texttt{MultiSPA}          & 31.47 & 50.68 & 13.88 & 0.07  & 8.75  & 0.28  & 15.22 & 0.54  & 17.09 & 0.07 \\
    \midrule
    \texttt{MultiSPA-KL} & \textbf{29.23} & 536.89 & \textbf{11.11} & 1.94  & \textbf{7.12} & 17.06 & \textbf{14.58} & 12.34 & \textbf{15.48} & 15.88 \\
    \midrule
    \texttt{MultiSPA-D\&S}      & 29.84 & 53.14 & 12.03 & 0.09  & \textbf{7.12} & 0.32  & 15.11 & 0.84  & \textbf{16.11} & 0.12 \\
    \midrule
    \texttt{Spectral-D\&S}      & \textbf{29.58} & 919.98 & 12.03 & 1.97  & \textbf{7.12} & 6.40  & 16.88 & 179.92 & 17.84 & 51.16 \\
    \midrule
    \texttt{TensorADMM} & N/A      &   N/A    & 12.03 & 2.74  & N/A     &   N/A   &  N/A    &  N/A    & 17.96 & 603.93 \\
    \midrule
    \texttt{MV-D\&S} & 30.02 & 3.20  & 12.03 & 0.02  & 7.25  & 0.07  & 16.02 & 0.28  & 15.86 & 0.04 \\
    \midrule
    \texttt{Minmax-entropy}    & 91.61 & 352.36 & \textbf{8.33} & 3.43  & 7.50  & 9.10  & \textbf{11.51} & 26.61 & 16.23 & 7.22 \\
    \midrule
    \texttt{EigenRatio} & 43.95 & 1.48  & 27.77 & 0.02  & 9.01  & 0.03  &  N/A     &  N/A     & N/A      & N/A \\
    \midrule
    \texttt{KOS}   & 51.95 & 9.98  & \textbf{11.11} & 0.01  & 39.75 & 0.03  & 42.93 & 0.31  & 31.84 & 0.13 \\
    \midrule
    \texttt{GhoshSVD} & 43.03 & 11.62 & 27.77 & 0.01  & 49.12 & 0.03  &  N/A    &  N/A    &  N/A     & N/A \\
    \midrule
    \texttt{Majority Voting}   & 34.85 &   N/A   & 21.29 &   N/A   & 10.31 &  N/A     & 26.93 &  N/A    & 17.91 & N/A \\
    \bottomrule
    \end{tabular}%
    }
  \label{tab:amazon}%
  \vspace{-.5cm}
\end{table}%
  \vspace{-.35cm}
\section{Conclusion}
  \vspace{-.35cm}
In this work, we have revisited the classic Dawid-Skene model for multi-class crowdsourcing.
We have proposed a second-order statistics-based approach that guarantees identifiability of the model parameters, i.e., the confusion matrices of the annotators and the label prior.
The proposed method naturally admits lower sample complexity relative to existing methods that utilize tensor algebra to ensure model identifiability.
The proposed approach also has an array of favorable features.
In particular, our framework enables a lightweight algebraic algorithm, which is reminiscent of the Gram-Schmitt-like SPA algorithm for nonnegative matrix factorization.
We have also proposed a coupled and constrained matrix factorization criterion that enjoys enhanced-identifiability, as well as an alternating optimization algorithm for handling the identification problem.
Real-data experiments show that our proposed algorithms are quite promising for integrating crowdsourced labeling.


\bibliographystyle{icml2019}

\clearpage
\appendix

{\bf Supplementary Materials for ``Crowdsourcing via Pairwise Co-occurrences: Identifiability and Algorithms''}

\section{Synthetic Data Experiments} \label{app:exp}
In the first experiment, we consider that $M=25$ annotators are available to annotate $N=10,000$ items, each belonging to one of $K=3$ classes. The true label for each item is sampled uniformly from $\{1,\dots,K\}$, i.e, the prior probability vector $\bm d$ is fixed to be $\bm d =[\nicefrac{1}{3},\nicefrac{1}{3},\nicefrac{1}{3}]^{\top}$. For generating the confusion matrices, two different cases are considered
\begin{itemize}
    \item \textbf{Case 1}: an annotator is chosen uniformly at random and is assigned an ideal confusion matrix, ie., an identity matrix ${\bm I}_3$. This ensures the assumption as given by Eq.\eqref{eq:sep} (or Eq.~\eqref{eq:anchor}).
    \item  \textbf{Case 2}: an annotator $m$ is chosen uniformly at random and its confusion matrix is made diagonally dominant such that ${\bm A}_m(k,k) > {\bm A}_m(k',k), \text{for}~ k',k \in \{1,\dots,K\},k\neq k'$. To achieve this, the elements of each column of ${\bm A}_m$ is drawn from a uniform distribution between 0 and 1. The columns are then normalized using their respective $\ell_1$-norms. After that, for each column, the elements are re-organized such that the corresponding diagonal entry is dominant in that column and then normalized with respect to $\ell_1$-norm. In this way, Eq. \eqref{eq:thm1} in Theorem~\ref{thm:spa} may be (approximately) satisfied.
\end{itemize}
 In both the cases, for the remaining annotators, the confusion matrices ${\bm A}_m$ are randomly generated; the elements are first drawn following the uniform distribution between 0 and 1, and then the columns are normalized with respect to the $\ell_1$-norm. Once ${\bm A}_m$'s are generated, the responses from each annotator $m$ for the items with true labels $g \in \{1,\dots,K\}$ are randomly chosen from $\{1,\dots,K\}$ using the probability distribution ${\bm A}_m(:,g)$. An annotator response for each item is retained for the estimation of ${\bm A}_m$  with probability $p \in (0,1]$. In other words, with probability $1-p$, each response is made 0. In this way, our simulated scenario is expected to mimic realistic situations where we have a combination of reliable and unreliable annotators, each labeling parts of the items. Using the generated responses, we construct $\widehat{ \bm R}_{m,\ell}$'s and then follow the proposed approach to identify the confusion matrices and the prior $\bm d$.

The accuracy of the estimation is measured using \textit{mean squared error} (MSE) defined as
\begin{align}
    \text{MSE} = \underset{\pi(k) \in \{1,\dots,K\}}{\text{min}} \frac{1}{K} \sum_{k=1}^K  \lVert {\bm A}_m(:,\pi(k))-\widehat{{\bm A}}_m(:,k)\rVert_2^2
\end{align}
where $\widehat{{\bm A}}_m$ is the estimate of ${\bm A}_m$ and $\pi(k)$'s are used to fix the column permutation.

The average (MSE) of the confusion matrices for various values of $p$ under the above mentioned cases are shown in Table \ref{tab:mse1} and Table \ref{tab:mse2} where the proposed methods, \texttt{MultiSPA} and \texttt{MultiSPA-KL} are compared with the baselines \texttt{Spectral-E\&M}, \texttt{TensorADMM} and \texttt{MV-D\&S} since these methods are also Dawid-Skene model identification approaches. As \texttt{MV-D\&S} becomes numerically unstable for smaller values of $p$, those results are not reported in the table. All the results are averaged from 10 trials. 

From the two tables, one can see that \texttt{MultiSPA} works reasonably well for both cases. As expected, it exhibits lower MSEs for case 1, since the condition in \eqref{eq:anchor} is perfectly enforced. Nevertheless, in both cases, using \texttt{MultiSPA} to initialize the \texttt{KL} algorithm identifies the confusion matrices to a very high accuracy.
It is observed that \texttt{MultiSPA-KL} outperforms the baselines in terms of the estimation accuracy ---which may be a result of using second order statistics.


\begin{table}[htbp]
  \centering
  \caption{ Average MSE of the confusion matrices ${\bm A}_m$ for case 1.}
    \begin{tabular}{|l|r|r|r|r|}
    \toprule
    \textbf{Algorithms} & \multicolumn{1}{l|}{$p=0.2$} & \multicolumn{1}{l|}{$p=0.3$} & \multicolumn{1}{l|}{$p=0.5$} & \multicolumn{1}{l|}{$p=1$} \\
    \midrule
    \texttt{MutliSPA} & 0.0184 & 0.0083 & 0.0063 & 0.0034 \\
    \midrule
    \texttt{MultiSPA-KL} & \textbf{0.0019} & \textbf{0.0009} & \textbf{0.0004} & \textbf{1.73E-04} \\
    \midrule
    \texttt{Spectral D\&S} & 0.0320 & 0.0112 & 0.0448 & 1.74E-04 \\
    \midrule
    \texttt{TensorADMM} & 0.0026 & 0.0011 & 0.0005 & 1.88E-04 \\
    \midrule
    \texttt{MV-D\&S} &    --   &    --   & 0.0173 & 1.84E-04 \\
    \bottomrule
    \end{tabular}%
   \label{tab:mse1}%
\end{table}%

\begin{table}[htbp]
  \centering
  \caption{Average MSE of the confusion matrices ${\bm A}_m$ for case 2.}
    \begin{tabular}{|l|r|r|r|r|}
    \toprule
    \textbf{Algorithms} & \multicolumn{1}{l|}{$p=0.2$} & \multicolumn{1}{l|}{$p=0.3$} & \multicolumn{1}{l|}{$p=0.5$} & \multicolumn{1}{l|}{$p=1$} \\
    \midrule
    \texttt{MutliSPA} & 0.0229 & 0.0188 & 0.0115 & 0.0102 \\
    \midrule
    \texttt{MultiSPA-KL} & \textbf{0.0029} & \textbf{0.0014} & \textbf{0.0005} & \textbf{1.67E-04} \\
    \midrule
    \texttt{Spectral D\&S} & 0.0348 & 0.0265 & 0.0391 & \textbf{1.67E-04} \\
    \midrule
    \texttt{TensorADMM} & 0.0031 & 0.0016 & 0.0006 & 1.93E-04 \\
    \midrule
    \texttt{MV-D\&S} &   --    &   --    & 0.0028 & 5.88E-04 \\
    \bottomrule
    \end{tabular}%
   \label{tab:mse2}%
\end{table}%

Under the same settings as in case 2, the true labels are estimated using the MAP/ML predictor as in \cite{traganitis2018blind} (in this case, ML and MAP are the same since the prior PMF is a uniform distribution). The classification error and the runtime of the crowdsourcing algorithms are computed and shown in Table \ref{tab:acc}.

\begin{table}[t]
  \centering
  \caption{Classification Error(\%) \& Averge run-time when ${\bm d} =[\frac{1}{3},\frac{1}{3},\frac{1}{3}]^{\top} $}
    \begin{tabular}{|l|r|r|r|r|}
    \toprule
    \textbf{Algorithms} & \multicolumn{1}{l|}{\textbf{$p=0.2$}} & \multicolumn{1}{l|}{\textbf{$p=0.3$}} & \multicolumn{1}{l|}{\textbf{$p=0.5$}} & \multicolumn{1}{l|}{\textbf{Run-time(sec)}} \\
    \midrule
    \texttt{MultiSPA}          & 37.24 & 26.39 & 19.21 & 0.049 \\
    \midrule
    \texttt{MultiSPA-KL} & \textbf{31.71} & \textbf{21.10} & \textbf{12.79} & 18.07 \\
    \midrule
    \texttt{MultiSPA-D\&S}      & \textbf{31.95} & \textbf{21.11} & \textbf{12.80} & 0.069 \\
    \midrule
    \texttt{Spectral-D\&S}      & 46.37 & 23.92 & 12.89 & 27.17 \\
    \midrule
    \texttt{TensorADMM} & 32.16 & 21.34 & 12.91 & 56.09 \\
    \midrule
    \texttt{MV-D\&S} & 66.91 & 57.92 & 13.09 & 0.096 \\
    \midrule
    \texttt{Minmax-entropy}    & 62.83 & 65.50  & 67.31  & 200.91 \\
    \midrule
    \texttt{KOS}   & 71.47 & 61.05 & 13.12 & 5.653 \\
    \midrule
    \texttt{Majority Voting}   & 67.57 & 68.37 & 71.39 & -- \\
    \bottomrule
    \end{tabular}%
   \label{tab:acc}%
\end{table}%

In the next experiment with case 2, the true labels are sampled with unequal probability. Specifically, ${\bm d}$ is set to be $[\frac{1}{6},\frac{2}{3},\frac{1}{6}]^{\top}$ with all other parameters and conditions same as in the first experiment. Using the MAP predictor, the true labels are estimated for the proposed algorithms for various values of $p$ and the results are shown in Table \ref{tab:acc1}. 
It can be inferred from the results that both the proposed algorithms \texttt{MultiSPA} and \texttt{MultiSPA-KL} grantee better classification accuracy when the true label distribution of the items is not balanced.

\begin{table}[t]
  \centering
  \caption{Classification Error(\%) \& Averge run-time when ${\bm d} =[\frac{1}{6},\frac{2}{3},\frac{1}{6}]^{\top} $}
    \begin{tabular}{|l|r|r|r|r|}
    \toprule
    \textbf{Algorithms} & \multicolumn{1}{l|}{\textbf{$p=0.2$}} & \multicolumn{1}{l|}{\textbf{$p=0.3$}} & \multicolumn{1}{l|}{\textbf{$p=0.5$}} & \multicolumn{1}{l|}{\textbf{Run-time(sec)}} \\
    \midrule
    \texttt{MultiSPA}          & \textbf{30.75} & \textbf{21.29} & \textbf{13.67} & 0.105 \\
    \midrule
    \texttt{MultiSPA-KL} & \textbf{23.19} & \textbf{16.62} & \textbf{10.13} & 18.93 \\
    \midrule
    \texttt{MultiSPA-D\&S}      & 40.12 & 32.1  & 21.46 & 0.122 \\
    \midrule
    \texttt{Spectral-D\&S}      & 56.17 & 49.41 & 39.17 & 28.01 \\
    \midrule
    \texttt{TensorADMM} & 34.17 & 25.53 & 11.97 & 152.76 \\
    \midrule
    \texttt{MV-D\&S} & 83.14 & 83.15 & 32.98 & 0.090 \\
    \midrule
    \texttt{Minmax-entropy}    & 83.04 & 63.08 & 74.29 & 232.82 \\
    \midrule
    \texttt{KOS}   & 70.79 & 67.55 & 78.00    & 6.19 \\
    \midrule
    \texttt{Majority Voting}   & 65.37 & 65.57 & 66.06 &--  \\
    \bottomrule
    \end{tabular}%
   \label{tab:acc1}%
\end{table}%

 In the next experiment, the effect of the number of annotators ($M$)  in the estimation accuracy of the confusion matrices is investigated. According to Theorem \ref{thm:Lm} and \ref{thm:SSM}, the proposed methods will benefit from the availability of more annotators (i.e., a larger $M$).  For $N=10,000$, $K=3$, ${\bm d} = [\frac{1}{6},\frac{2}{3},\frac{1}{6}]^{\top}$, $p=0.5$ and the true confusion matrices ${\bm A}_m$ being generated as in case 2, the MSEs under various values of $M$ are plotted in Figure \ref{mvar}. One can see that \texttt{MultiSPA-KL} achieves better accuracy relative to \texttt{MultiSPA} under the same $M$'s, which corroborates our results in Theorem \ref{thm:SSM}. 
  
 \begin{figure}	
\centering
\includegraphics[scale=0.5]{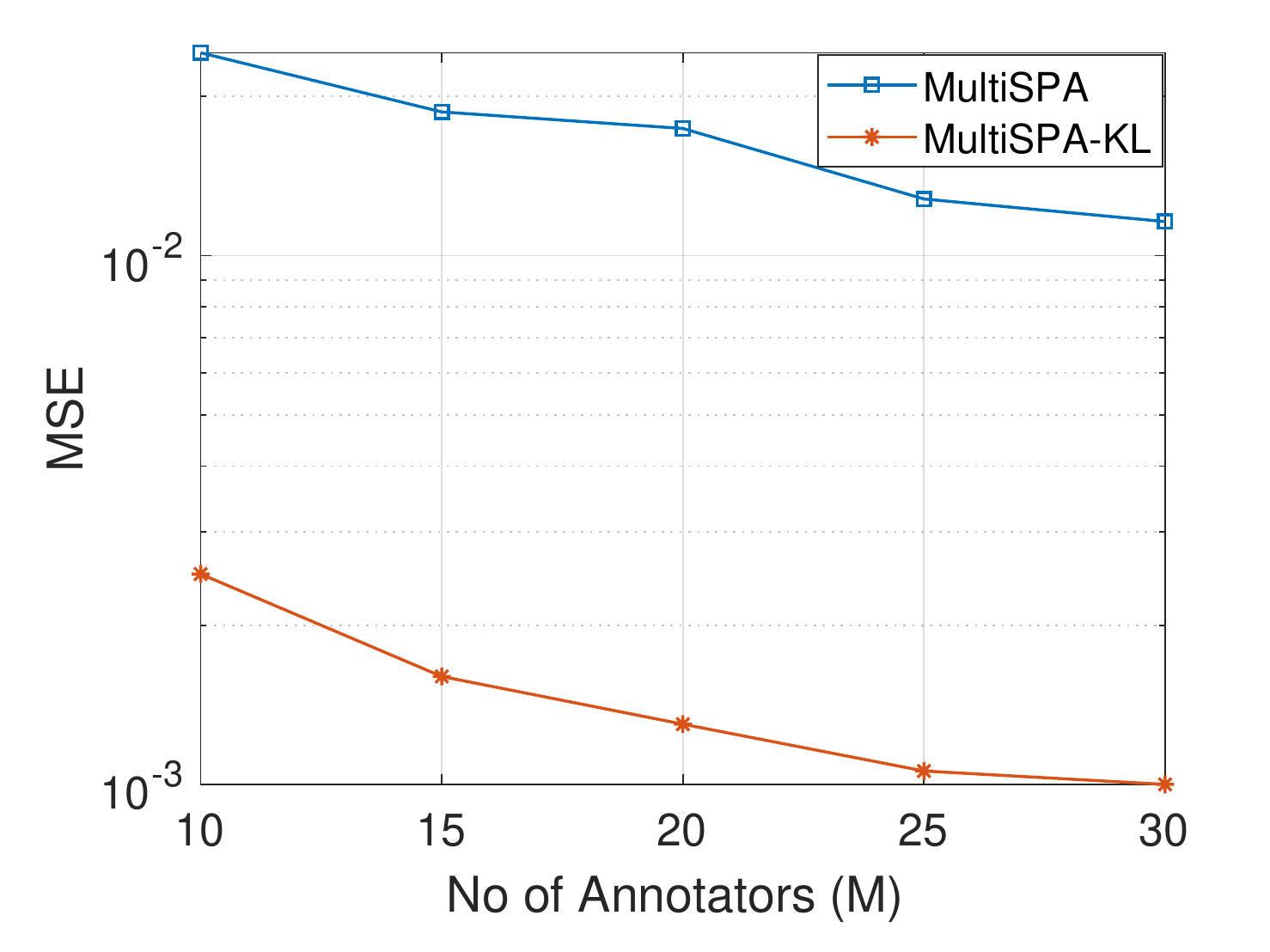}
\caption{MSE of the confusion matrices for various values of $M$}
\label{mvar}
\end{figure}

\section{More Details on UCI and AMT Dataset Experiments} \label {real_dataset}

\noindent
{\bf UCI data.} The details of the UCI datasets employed in the real data experimemts is given in Table \ref{tab:real-data}.
To be more specific, the Adult dataset predicts the income of a person into $K=2$ classes based on 14 attributes.  
The Mushroom dataset has 22 attributes of certain variations of mushrooms and the task there predicts either `edible' or `poisonous'. The Nursery dataset predicts applications to one of the 4 categories based on 8 attributes of the financial and social status of the parents. 

\begin{table}[t]
  \centering
   \caption{Details of UCI Datasets.}
    \begin{tabular}{|l|r|r|r|r|}
    \toprule
    UCI dataset name  & \multicolumn{1}{l|}{\# classes} & \multicolumn{1}{l|}{\# items} & \multicolumn{1}{l|}{\# annotators}  \\
    \midrule
    Adult & 2     & 7017  & 10    \\
    \midrule
    Mushroom & 2     & 6358  & 10   \\
    \midrule
    Nursery & 4     & 3575  & 10    \\
    \bottomrule
    \end{tabular}%
  \label{tab:real-data}%
\end{table}%

The proposed methods and the baselines are compared in terms of runtime for various datasets and the results are reported in Table \ref{tab:realtime}. All the results are averaged from 10 different trials.

\begin{table}[t]
  \centering
  \caption{Average runtime (sec) for UCI datset experiments.}
    \begin{tabular}{|l|r|r|r|}
    \toprule
    \textbf{Algorithms} & \multicolumn{1}{l|}{\textbf{Nursery}} & \multicolumn{1}{l|}{\textbf{Mushroom}} & \multicolumn{1}{l|}{\textbf{Adult}} \\
    \midrule
    \texttt{MultiSPA} & 0.021 & 0.012 & 0.018 \\
    \midrule
    \texttt{MultiSPA-KL} & 1.112 & 0.663 & 0.948 \\
    \midrule
    \texttt{MultiSPA-D\&S} & 0.035 & 0.027 & 0.027 \\
    \midrule
    \texttt{Spectral-D\&S} & 10.09 & 0.496 & 0.512 \\
    \midrule
    \texttt{TensorADMM} & 5.811 & 0.743 & 4.234 \\
    \midrule
    \texttt{MV-D\&S} & 0.009 & 0.007 & 0.008 \\
    \midrule
    \texttt{Minmax-entropy}    & 19.94 & 2.304 & 6.959 \\
    \midrule
    \texttt{EigenRatio} &   --    & 0.005 & 0.007 \\
    \midrule
    \texttt{KOS}   & 0.768 & 0.085 & 0.118 \\
    \midrule
    \texttt{Ghosh-SVD} &  --     & 0.081 & 0.115 \\
    \bottomrule
    \end{tabular}%
  \label{tab:realtime}%
\end{table}%

\begin{table}[h!]
	\centering
	\caption{AMT Dataset description.}
	\begin{tabular}{|l|r|r|r|r|}
		\toprule
		\textbf{Dataset} & \multicolumn{1}{l|}{\textbf{\# classes}} & \multicolumn{1}{l|}{\textbf{\# items}} & \multicolumn{1}{l|}{\textbf{\# annotators}} & \multicolumn{1}{l|}{\textbf{\# annotator labels}} \\
		\midrule
		Bird  & 2     & 108   & 30    & 3240 \\
		\midrule
		RTE   & 2     & 800   & 164   & 8,000 \\
		\midrule
		TREC  & 2     & 19,033 & 762   & 88,385 \\
		\midrule
		Dog   & 4     & 807   & 52    & 7,354 \\
		\midrule
		Web   & 5     & 2,665 & 177   & 15,567 \\
		\bottomrule
	\end{tabular}%
	\label{tab:amt}%
\end{table}%

\noindent
{\bf AMT data.} The Amazon Mechanical Turk (AMT) datasets used in our crowdsourcing data experiments is given in Table \ref{tab:amt}.
Specifically, the tasks involving the Bird dataset \cite{welinder2010multidimensional}, the RTE dataset  \cite{snow2008cheap}, and the TREC dataset \cite{lease2011Overview}, are binary classification tasks. The tasks associated with the Dog dataset \cite{deng2009imagenet} and the web dataset \cite{zhou2014aggregated} are multi-class tasks (i.e., 4 and 5 classes, respectively). 

We would like to add one remark regarding the two-stage approaches that involving an initial stage and a refinement stage (e.g., \texttt{Spectral-D\&S}, \texttt{MV-D\&S}, and \texttt{MultiSPA-KL}).
Due to very high sparsity of the annotator responses in most of the AMT data, the estimated confusion matrices from the first stage may contain many zero entries, which may sometimes lead to numerical issues in the second stage, as observed in \cite{zhang2014spectral}.
In our experiments, we follow an empirical thresholding strategy proposed in \cite{zhang2014spectral}. Specifically, the confusion matrix entries that are smaller than a threshold $\Delta$ are reset to $\Delta$ and the columns are normalized before initialization. In our experiments, we use $\Delta = 10^{-6}$ for most of the cases except the extremely large dataset TREC, which enjoys better performance of all methods using $\Delta=10^{-5}$.

\section{Algorithm for Criterion~\eqref{eq:coupled}} \label{algo}

In this section, the \texttt{MultiSPA-KL} algorithm is discussed in detail. To implement the identification criterion in \eqref{eq:coupled}, we lift the constraint \eqref{eq:R} and employ the following coupled matrix factorization cirterion:
\begin{subequations}\label{eq:KL}
	\begin{align}
\minimize_{ \{\bm A_m\}_{m=1}^M,~{\bm D}}&~\sum_{m,\ell} {\sf KL}\left( \widehat{ \bm R}_{m,\ell}||\bm A_m\bm D\bm A_\ell^\T \right),\\
{\rm subject~to:}&~ {\bm 1}^\top\bm A_m =\bm 1^\top,~\bm A_m\geq \bm 0,~\bm 1^\top\bm d=1,~\bm d\geq \bm 0,
	\end{align}
\end{subequations}
where $\bm D={\rm Diag}(\bm d)$ and the Kullback-Leibler (KL) divergence is employed as the distance measure. The reason is that ${\bm R}_{m,\ell}$ is a joint PMF of two random variables, and the KL-divergence is the most natural distance measure under such circumstances.
Problem~\eqref{eq:KL} is a nonconvex optimization problem, but can be handled by a simple alternating optimization procedure. 

Specifically, we propose to solve the following subproblems cyclically:
\begin{subequations}
	\begin{align}
	\bm A_m &\leftarrow \arg\min_{{\bm 1}^\top\bm A_m =\bm 1^\top,~\bm A_m\geq \bm 0}~\sum_{\ell\in{\cal S}_m} {\sf KL}\left(  \widehat{\bm R}_{m,\ell}||\bm A_m\bm D\bm A_\ell^\T \right)\label{eq:Aupdate}\\
	\bm d &\leftarrow \arg\min_{{\bm 1}^\top\bm d=1,~\bm d\geq \bm 0}~\sum_{\ell\in{\cal S}_m} {\sf KL}\left(  \widehat{\bm R}_{m,\ell}||\bm A_m\bm D\bm A_\ell^\T \right)\label{eq:dupdate}
	\end{align}
\end{subequations}
where ${\cal S}_m$ denotes the index set of $\ell$'s such that $\bm R_{m,\ell}$ is available. 
Both of the above problems are convex optimization problems, and thus can be effectively solved via a number of off-the-shelf optimization algorithms, e.g., ADMM \cite{huang2016flexible} and mirror descent \cite{arora2012practical}. 
The detailed summarized algorithm is in Algorithm~\ref{algo:MultiSPA-KL}.
The alternating optimization algorithm is also guaranteed to converge to a stationary point under mild conditions \cite{bertsekas1999nonlinear,razaviyayn2013unified}.

\begin{algorithm}[h]
	\caption{\texttt{MultiSPA-KL}}\label{algo:MultiSPA-KL}
	\begin{algorithmic}
		\STATE {\bfseries Input:} Annotator Responses $\{X_m(\bm f_n)\}$.
		\STATE {\bfseries Output:} $\widehat{\bm A}_m$ for $m=1,\ldots,M$, $\widehat{\bm d}$.
		\STATE Estimate second order statistics $\widehat{\bm R}_{m,\ell }$;
		\STATE get initial estimates of $\{\widehat{\bm A}_m\}$ using \texttt{MultiSPA}
		\FOR{$t=1$ {\bfseries to} $\texttt{MaxIter}$}
		\FOR{$m=1$ {\bfseries to} $M$}
		\STATE  update $\bm A_m \leftarrow \eqref{eq:Aupdate}$;
			\ENDFOR
		\STATE update $\bm d\leftarrow\eqref{eq:dupdate}$;
		\ENDFOR
	\end{algorithmic}
\end{algorithm}

Note that this coupled factorization formulation bears some resemblance to the coupled tensor factorization formulation in \cite{traganitis2018blind}.
However, the two are very different in essence. The formulation in \cite{traganitis2018blind} relies on the third-order statistics to establish identifiability, while the formulation in \eqref{eq:KL} establishes identifiability using nonnegativity of the confusion matrices and the prior.
The KL-divergence based fitting criterion also fits the statistical learning problem better than the least squares based criterion in \cite{traganitis2018blind}.

\section{Estimation of Prior Probability Vector} \label{prior}
In this section, we discuss different methods to estimate the prior probability vector ${\bm d}$ once the confusion matrices are estimated via \texttt{MultiSPA} algorithm.

It is to be noted that the SPA-estimated $\widehat{\bm A}_m$ is up to column permutation, even if there is no noise, i.e.,
$\widehat{\bm A}_m = \bm A_m \bm \Pi_m $
in the best case.
Since our algorithm runs SPA separately for different $\overline{\bm Z}_m$'s, the permutation matrices resulted by each run of SPA need not to be identical; i.e., it is highly likely that
$\bm \Pi_\ell \neq \bm \Pi_m$ for $m\neq \ell.$
To estimate the prior PMF $\bm d$, one will need to use estimators such as $$\widehat{\bm D}=\widehat{\bm A}_m^{-1}\bm R_{m,\ell}(\widehat{\bm A}_\ell^\top)^{-1},$$
which cannot be applied before the permutation mismatch is fixed. In practice, the mismatch can be removed by a number of simple methods. For example, if annotator $\ell$ has co-labeled data with annotator $m$, then $\tilde{\bm A_\ell}=\bm A_\ell\bm \Pi_m$ can be estimated from $\bm R_{m,\ell}$ via
$  \tilde{\bm A_\ell} = \widehat{\bm A}_m^{-1}\bm R_{m,\ell}. $
We also have $\widehat{\bm A}_\ell=\bm A_\ell\bm \Pi_\ell$ estimated from $\overline{\bm Z}_\ell$.
Using a permutation matching algorithm, e.g., the Hungarian algorithm \cite{jonker1986improving}, one can easily remove the permutation mismatch between $\widehat{\bm A}_\ell$ and $\tilde{\bm A}_\ell$.
Another more heuristic yet more efficient way is to rearrange the columns of $\widehat{\bm A}_m$ so that it is diagonally dominant---this makes a lot of sense if one believes that all the annotators are reasonably trained.

\section{Geometry of The Sufficiently Scattered Condition}\label{app:convex_geometry}
In this section, we present more discussion on the sufficiently scattered condition that is used in Theorem~\ref{thm:couple}.
To simplify the notation, we omit the superscript of $\bm H^{(i)}$ for $i=1,2$ and use $\bm H$ to denote these two matrices.
The sufficiently scattered condition is geometrically intuitive. 
The key to understand this condition is the second-order cone ${\cal C}$, which is shown in Fig.~\ref{fig:coneplot}.
This cone is very special since it is tangent to all the facets of the nonnegative orthant.

A case where $\bm H\in\mathbb{R}^{N\times K}$ satisfies the sufficiently scattered condition is plotted in Fig.~\ref{fig:scattered}.
One can see that the sufficiently scattered condition is much more relaxed compared to the condition that enables SPA (cf. Fig.~\ref{fig:separable}).
In order to apply SPA to $\overline{\bm Z}_m$, one needs that there are rows in $\overline{\bm H}_m$ that attain the extreme rays of the nonnegative orthant.

\begin{figure}[h]
	\centering
	\includegraphics[width=0.6\linewidth]{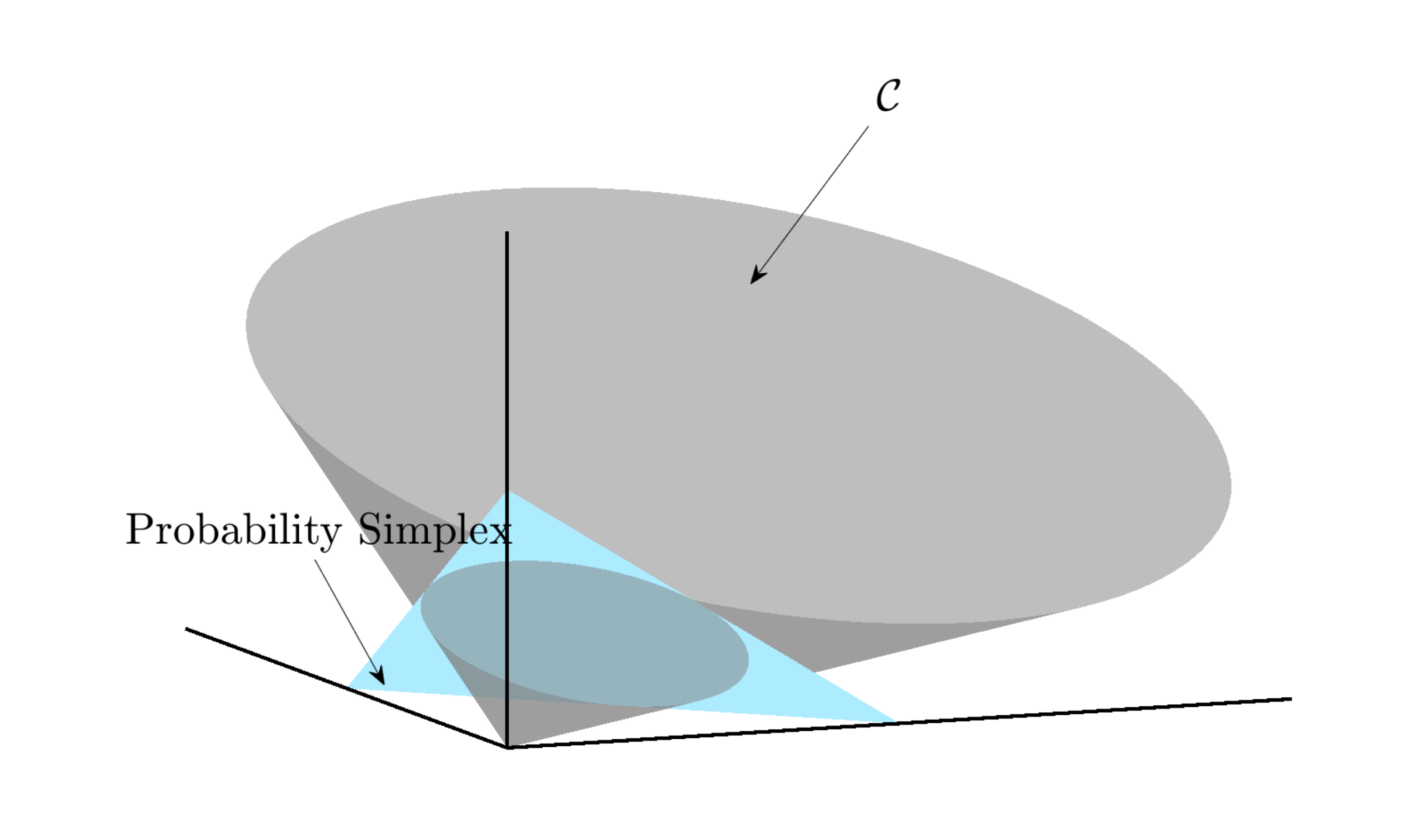}
	\caption{Illustration of the cone ${\cal C}$ in an $3$-dimensional space.}
	\label{fig:coneplot}
\end{figure}

\begin{figure}[h]
	\centering
	\begin{tikzpicture}
\filldraw[fill=pink]
( 75:1.4142) -- (105:1.4142) -- (195:1.4142) -- (225:1.4142) -- 
(-45:1.4142) -- (-15:1.4142) -- ( 75:1.4142);

\fill[fill=blue] ( 75:1.4142) circle(2pt);
\fill[fill=blue] (105:1.4142) circle(2pt);
\fill[fill=blue] (195:1.4142) circle(2pt);
\fill[fill=blue] (225:1.4142) circle(2pt);
\fill[fill=blue] (-45:1.4142) circle(2pt);
\fill[fill=blue] (-15:1.4142) circle(2pt);


\fill[fill=blue] ( 50:1) circle(2pt);
\fill[fill=blue] ( 80:1.2) circle(2pt);
\fill[fill=blue] (120:.9) circle(2pt);
\fill[fill=blue] (177:1) circle(2pt);
\fill[fill=blue] (200:1.2) circle(2pt);
\fill[fill=blue] (-99:.7) circle(2pt);
\fill[fill=blue] (-30:1.2) circle(2pt);

\fill[fill=blue] (  0:.6) circle(2pt);
\fill[fill=blue] (170:.3) circle(2pt);

\draw[very thick] (90:2cm) -- (210:2cm) -- (-30:2cm) -- (90:2cm);
\draw[very thick] (0,0) circle (1cm);
\end{tikzpicture}
	\caption{A case where $\bm H$ satisfies the sufficiently scattered condition. The inner circle corresponds to ${\cal C}$, the dots correspond to $\bm H(q,:)$'s, and the triangle corresponds to the nonnegative orthant. The shaded region is $\cone{\bm H^\T}$.  }\label{fig:scattered}
\end{figure}
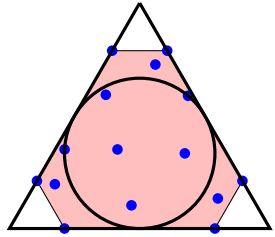

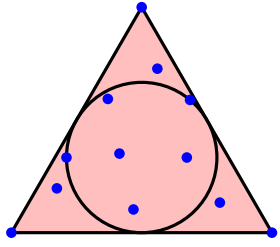
\begin{figure}[h]
	\centering
	\begin{tikzpicture}
\fill[fill=pink] (90:2) -- (210:2) -- (-30:2);

\draw[very thick] (90:2) -- (210:2) -- (-30:2) -- (90:2);
\draw[very thick] (0,0) circle (1cm);

\fill[fill=blue] ( 90:2) circle(2pt);
\fill[fill=blue] (210:2) circle(2pt);
\fill[fill=blue] (-30:2) circle(2pt);

\fill[fill=blue] ( 50:1) circle(2pt);
\fill[fill=blue] ( 80:1.2) circle(2pt);
\fill[fill=blue] (120:.9) circle(2pt);
\fill[fill=blue] (180:1) circle(2pt);
\fill[fill=blue] (200:1.2) circle(2pt);
\fill[fill=blue] (-99:.7) circle(2pt);
\fill[fill=blue] (-30:1.2) circle(2pt);

\fill[fill=blue] (  0:.6) circle(2pt);
\fill[fill=blue] (170:.3) circle(2pt);

\end{tikzpicture}
	\caption{A case where $\overline{\bm H}_m$ satisfies the condition for applying SPA (aka. the separability condition). The inner circle corresponds to ${\cal C}$, the dots correspond to $\overline{\bm H}_m(q,:)$'s, and the triangle corresponds to the nonnegative orthant. The shaded region is $\cone{ \overline{\bm H}_m^\T}$. }\label{fig:separable}
\end{figure}

\section{Proof of Theorem~\ref{thm:spa}}

\subsection{Identification Theory of SPA}\label{app:spa}
To understand Theorem~\ref{thm:spa}, let us start with a noisy nonnegative matrix factorization (NMF) model
\begin{equation}\label{eq:sepNMF}
\bm X= \bm W\bm H^\T +\bm N, 
\end{equation}
where ${\bm W}\in\mathbb{R}^{M\times K}$, $\bm H\in\mathbb{R}^{N\times K}$, $\bm W\geq \bm 0$ and $\bm H\geq \bm 0$, and $\bm N$ represents the noise.
Also assume that ${\rm rank}(\bm W)=K$ and ${\bm H}=\bm \Pi\begin{bmatrix}
\bm I_K\\
\bm H^\ast
\end{bmatrix};$ i.e., there exists $\varLambda=\{q_1,\ldots,q_K\}$
such that $\bm H(\varLambda,:)=\bm I$.
Also assume that
$   \bm H\bm 1=\bm 1. $

The SPA algorithm under this model is as follows \cite{Gillis2012,fu2014self,arora2012practical,VMAX}:
	\begin{shaded}
	\begin{subequations}\label{eq:spa2}
		\begin{align*}
		\widehat{q}_1& = \arg\max_{q\in\{1,\ldots,N\}}~\left\|{\bm X}(:,q)\right\|_2^2\\
		\widehat{q}_k& = \arg\max_{q\in\{1,\ldots,N\}}~\left\|\bm P^\perp_{\widehat{\bm W}(:,1:k-1)}{\bm X}(:,q)\right\|_2^2,~k>1.
		\end{align*}
	\end{subequations}
	where $$\widehat{\bm W}(:,1:k-1)=[\bm X(:,\widehat{q}_1),\ldots,\bm X(:,\widehat{q}_{k-1})]$$ collects all the previously estimated columns of $\bm W$ and $\bm P^\perp_{\widehat{\bm W}(:,1:k-1)}$ is a projector onto the orthogonal complement of ${\rm range}(\widehat{\bm W}(:,1:k-1))$.
\end{shaded}

When there is no noise, it was shown in the literature that SPA readily identifies $\varLambda$ \cite{Gillis2012,arora2012practical}.
To see this, consider
\begin{align*}
\|\bm X(:,q)\|_2 & = \left\|\sum_{k=1}^K\bm W(:,k)\bm H(q,k)\right\|_2\leq \sum_{k=1}^K\left\|\bm W(:,k)\bm H(q,k)\right\|_2\\
&= \sum_{k=1}^K\bm H(q,k)\left\|\bm W(:,k)\right\|_2\leq \max_{k=1,\ldots,K}\left\|\bm W(:,k)\right\|_2,
\end{align*}
where the two equalities hold simultaneously if and only if $\bm H(q,:)=\bm e_k^\T$ for a certain $k$, i.e., $q\in\varLambda$.
After identifying the first index $\widehat{q}_1$ in $\varLambda$, then by projecting all the data column onto the the orthogonal complement of $\bm X(:,\widehat{q}_1)$, the same $\bm W(:,k)$ will not come up again.
Hence, $K$ steps of SPA identifies the whole $\bm W$.

A salient feature of SPA is that it is provably robust to noise. To be specific, Gillis and Vavasis have shown that:
\begin{Lemma}\label{lem:gillis} \cite{Gillis2012}
	Under the described NMF model, assume that $\| \bm{N}(:,l) \|_2 \leq \delta$ for all $l$. If the below holds:
	\begin{equation*}
	\begin{aligned}
	\delta \le \sigma_{\rm min}(\bm{W})\text{\rm min}\left(\frac{1}{2\sqrt{K-1}},\frac{1}{4}\right)\left(1+80\kappa^2(\bm{W})\right)^{-1},
	\label{n2}
	\end{aligned}
	\end{equation*}
	then, SPA identifies an index set $\widehat{\varLambda}=\{\widehat{q}_1,\ldots\widehat{q}_K \}$ such that
	\begin{equation*}
	\begin{aligned}	
	\max_{1 \leq j \leq K} \min_{\widehat{q}_k \in \widehat{\varLambda}} \left\|\bm{W}(:,j) - \bm{X}(:,\widehat{q}_k)\right\|_2 \leq \delta\big(1+80\kappa^2(\bm{W})\big)
	\end{aligned}
	\end{equation*}
	where $\kappa(\bm{W}) = \frac{\sigma_{\max}(\bm{W})}{\sigma_{\min}(\bm{W})}$ is the condition number of $\bm{W}$. 
\end{Lemma}
\color{black}

\subsection{Proof of The Theorem} \label{app:theorem1}

Since $\bm R_{m,\ell}$ is obtained by sample averaging of a finite number of pairwise co-occurrences, the estimated $\widehat{\bm R}_{m,\ell } $ is always noisy; i.e., we have
\begin{align}
    \widehat{\bm R}_{m,\ell } = \bm R_{m,\ell} +{\bm N}_{m,\ell},
\end{align}
where the noise matrix ${\bm N}_{m,\ell}$ has same dimension as $\widehat{\bm R}_{m,\ell }$ or $\bm R_{m,\ell}$ and its norm can be bounded by Lemma \ref{lem:anandkumar}.

\begin{Lemma}\label{lem:anandkumar} \cite{anandkumar2014tensor}
	Let $\delta \in (0,1)$ and let $\widehat{\bm R}_{m,\ell }$ be the empirical average of $S$ independent co-occurrences of random variables $X_m$ and $X_{\ell}$ where $X_m,X_{\ell} \in \{1,\dots,K\}$, then the following holds 
	\begin{equation*}
	\begin{aligned}
	{\sf Pr}\bigg[\|\widehat{\bm R}_{m,\ell }-{\bm R}_{m,\ell }\|_F = \|{\bm N}_{m,\ell }\|_F \le \frac{1+\sqrt{{\rm ln}(1/\delta)}}{\sqrt{S}}\bigg] \ge 1-\delta
	\label{n2}
	\end{aligned}
	\end{equation*}
\end{Lemma}
Using the estimates $\widehat{\bm R}_{m,\ell }$, $\widehat{\bm Z}_m$ is constructed according to \eqref{eqX} (with ${\bm R}_{m,\ell }$'s replaced by $\widehat{\bm R}_{m,\ell }$'s). The columns of $\widehat{\bm Z}_m$  are normalized before performing $\texttt{MultiSPA}$, essentially normalizing the columns of $\widehat{\bm R}_{m,\ell }$. 
Normalization complicates the analysis since the demonstrators used in this step are also noisy.
We derive
Lemma \ref{lem:normalization} to characterize the noise bound after column normalization.   

\begin{Lemma}\label{lem:normalization} 
	Assume that there exists at least $S$ joint responses from each of the annotator pairs $m,\ell$.  Let $\eta \in (0,1)$. If $\|\widehat{\bm R}_{m,\ell }(:,k)\|_1 \ge \eta$ and $ \|{\bm N}_{m,\ell }(:,k)\|_1 < \|{\bm R}_{m,\ell }(:,k)\|_1,  \forall k \in \{1,\dots,K\}, \forall m\neq \ell$, then with probability greater than $1-\delta$, the below holds $\forall k \in \{1,\dots,K\}, \forall m\neq \ell$,
	\begin{equation*}
	\begin{aligned}
	    \frac{\widehat{\bm R}_{m,\ell }(:,k)}{\|\widehat{\bm R}_{m,\ell }(:,k)\|_1} = \frac{{\bm R}_{m,\ell }(:,k)}{\|{\bm R}_{m,\ell }(:,k)\|_1} + {\overline{\bm N}}_{m,\ell }(:,k)
	\label{n2}
	\end{aligned}
	\end{equation*}	
	where $\|{\overline{\bm N}}_{m,\ell }(:,k)\|_2 \le  \frac{2\sqrt{K}(1+\sqrt{{\rm ln}(1/\delta)})}{\sqrt{S}\eta}$.
\end{Lemma}	
\begin{proof}
	For simpler representation, let us assign ${\bm x} := {\bm R}_{m,\ell }(:,k)$, $\widehat{\bm x} := \widehat{{\bm R}}_{m,\ell }(:,k)$ and ${\bm n} :=\widehat{{\bm R}}_{m,\ell }(:,k)-{\bm R}_{m,\ell }(:,k) = {\bm N}_{m,\ell }(:,k) $.
	
	Let ${\bm x}=[x_1,\dots,x_K]^{\top}$ and ${\bm n}=[n_1,\dots,n_K]^{\top}$. Note that ${\bm x} \ge 0$ and ${\bm x} +{\bm n} \ge 0$---since $\bm x$ is a legitimate PMF and ${\bm x} +{\bm n}$ is averaged from co-occurrence counts. Then, we have
	\begin{align*}
	     \frac{{\widehat{\bm x}}}{\|{\widehat{\bm x}}\|_1} = \frac{{\bm x} +{\bm n}}{\|{\bm x} +{\bm n}\|_1} &= \frac{{\bm x} +{\bm n}}{\sum_{i} x_i +n_i}
	    =\frac{{\bm x} +{\bm n}}{\sum_{i} x_i +\sum_{i} n_i}\\
	    &=\frac{{\bm x} +{\bm n}}{\sum_{i} x_i\left(1 +\frac{\sum_{i} n_i}{\sum_{i} x_i}\right)}.
	\end{align*}
	Let $\mu =\frac{\sum_{i} n_i}{\sum_{i} x_i}$ . Using the assumption $\|{\bm N}_{m,\ell }(:,k)\|_1 < \|{\bm R}_{m,\ell }(:,k)\|_1$, then $|\mu| < 1 $, From this,
	\begin{align}
	    \frac{{\bm x} +{\bm n}}{\|{\bm x} +{\bm n}\|_1} &=\frac{({\bm x} +{\bm n})(1 +\mu)^{-1}}{\sum_{i} x_i} \nonumber
	    = \frac{({\bm x} +{\bm n})}{\sum_{i} x_i}(1 -\mu+{\mu}^2-{\mu}^3+\dots ) \nonumber\\
	    &= \frac{{\bm x}}{\sum_{i} x_i}-\mu\frac{{\bm x}}{\sum_{i} x_i}(1 -\mu+{\mu}^2-{\mu}^3+\dots )+\frac{{\bm n}}{\sum_{i} x_i}(1 -\mu+{\mu}^2-{\mu}^3+\dots ) \nonumber\\
	    &=\frac{{\bm x}}{\sum_{i} x_i}+{\frac{{\bm n}}{(1+\mu)\sum_{i} x_i}-\frac{\mu{\bm x}}{(1+\mu)\sum_{i} x_i}} \nonumber\\
	    &= \frac{{\bm x}}{\|{\bm x}\|_1}+\underbrace{\frac{{\bm n}-\mu{\bm x}}{(1+\mu)\sum_{i} x_i}}_{\Gamma} \label{eqxn}
	\end{align}	
	Now let us bound the term $\Gamma$:
	\begin{align}
	    \|\Gamma\|_1 &:= \left\Vert\frac{{\bm n}-\mu{\bm x}}{(1+\mu)\sum_{i} x_i}\right\Vert_1 \nonumber
	    \le \frac{\|{\bm n}\|_1}{\|{\bm x} +{\bm n}\|_1}+\frac{\|{\sum_{i} n_i}\|}{\|{\bm x} +{\bm n}\|_1} \nonumber\\
	    &\le 2\frac{\|{\bm n}\|_1}{\|{\bm x} +{\bm n}\|_1} \nonumber\\
	    &\le 2\frac{\|{\bm n}\|_1}{\eta}. \label{eqnn}
	\end{align}

The first and second inequalities are due to Cauchy-Schwartz ineqality and the last inequality is by  $\|\widehat{\bm R}_{m,\ell }(:,k)\|_1 \ge \eta$. 

From Lemma \ref{lem:anandkumar}, with probability greater than $1-\delta$, the below holds,
\begin{align*}
 \sum_{k=1}^{K}\|{\bm N}_{m,\ell }(:,k)\|^2_2 &=\|{\bm N}_{m,\ell }\|^2_F \le \frac{(1+\sqrt{{\rm ln}(1/\delta)})^2}{S} .
\end{align*}
By norm equivalence, $\frac{\|{\bm N}_{m,\ell }(:,k)\|_1}{\sqrt{K}} \le \|{\bm N}_{m,\ell }(:,k)\|_2$, Therefore,
\begin{align}
\sum_{k=1}^{K}\|{\bm N}_{m,\ell }(:,k)\|^2_1 \le \frac{K(1+\sqrt{{\rm ln}(1/\delta)})^2}{S}  \nonumber\\  
\implies \|{\bm N}_{m,\ell }(:,k)\|^2_1 \le \frac{K(1+\sqrt{{\rm ln}(1/\delta)})^2}{S}, \quad \forall k \nonumber\\
\implies \|{\bm N}_{m,\ell }(:,k)\|_1 \le \frac{\sqrt{K}(1+\sqrt{{\rm ln}(1/\delta)})}{\sqrt{S}}, \quad \forall k \label{eqN1}
\end{align}
From (\ref{eqnn}) and (\ref{eqN1}),
\begin{align*}
    \|\Gamma\|_1 = \|{\overline{\bm N}}_{m,\ell }(:,k)\|_1 &\le  \frac{2\sqrt{K}(1+\sqrt{{\rm ln}(1/\delta)})}{\sqrt{S}\eta}\\
\end{align*}

By norm equivalence, $\|{\overline{\bm N}}_{m,\ell }(:,k)\|_2 \le \|{\overline{\bm N}}_{m,\ell }(:,k)\|_1$, then
\begin{align*}
    \|{\overline{\bm N}}_{m,\ell }(:,k)\|_2 &\le  \frac{2\sqrt{K}(1+\sqrt{{\rm ln}(1/\delta)})}{\sqrt{S}\eta}\\
\end{align*}
This completes the proof of Lemma~\ref{lem:normalization}.
\end{proof}

With the above lemmas, we are ready to characterize the accuracy of applying SPA to identify the Dawid-Skene model given the assumptions in Eq.~\eqref{eq:thm1}.

Eq.~\eqref{eq:thm1} indicates that there exits a set of indices $\varLambda_q=\{q_1,\dots, q_K\}$ such that
\begin{equation}
\overline{\bm{H}}_m(\varLambda_q,:) = \bm{I}_{K} + \bm{E}
\end{equation}
where $\bm{I}_{K}$ is the identity matrix of size $K$ and $\bm{E}$ is the error matrix with $\text{max}_{j} | \bm{E}(l,j)| = \|\bm{E}(l,:)\|_{\infty} \le \epsilon$. By norm equivalence, we have $\|\bm{E}(l,:)\|_{2} \le \sqrt{K} \bm{E}(l,:)\|_{\infty} \le \sqrt{K} \epsilon$.

Without loss of generality, let us assume $\varLambda_q=\{1,\ldots,K\}$ and 
\[  \overline{\bm H}_m =\begin{bmatrix}
\bm I_K+\bm E\\\bm{H}_m^{*}
\end{bmatrix}. \]
Now we have, 
\begin{align}
\overline{\bm{Z}}_m 
&=\bm{A}_m \overline{\bm H}^{\top}_m +{\overline{\bm N}}\\
&= \bm{A}_m [\bm{I}_{K} + \bm{E}^\top \ \ \ (\bm{H}_m^{*})^\top]+{\overline{\bm N}} \\
&= \bm{A}_m [\bm{I}_{K}  \ \ \ (\bm{H}_m^{*})^\top] + [\bm{A}_m \bm{E}^\top \ \ \bm{0}]+{\overline{\bm N}}
\end{align}
where ${\overline{\bm N}}= \begin{bmatrix} 
    \overline{\boldsymbol{N}}_{m,m_1} ,\ldots,\overline{\boldsymbol{N}}_{m,m_{T(m)}}
    \end{bmatrix}$ and the zero matrix $\bm{0}$ has the same size as that of $\bm{H}_m^{*}$

This model is similar to the noisy NMF model, i.e., $\bm{X} = \bm{WH}^\top + \bm{N}$, where the noise matrix $\bm{N} = [\bm{A}_m \bm{E}^\top \ \ \bm{0}_m^*]+{\overline{\bm N}}$.
To be specifc, we have 
\begin{equation}
\bm{N}(:,l) = \bm{A}_m\bm{E}(l,:)^\T+{\overline{\bm N}}(:,l)
\end{equation}
Therefore, one can see that
\begin{align}
\| \bm{N}(:,l) \|_2 &= \|\bm{A}_m\bm{E}(l,:)^\T+{\overline{\bm N}}(:,l)\|_2 \\
& \le \|\bm{A}_m\|_2 \|\bm{E}(l,:)\|_2+ \|{\overline{\bm N}}(:,l)\|_2 \label{cauchy}\\
& \le \sigma_{\rm max}(\bm{A}_m)\sqrt{K}\epsilon + \frac{2\sqrt{K}(1+\sqrt{{\rm ln}(1/\delta)})}{\sqrt{S}\eta} \label{n1}
\end{align}
where \eqref{cauchy} is by the Cauchy-Schwartz inequality, \eqref{n1} is by Lemma \ref{lem:normalization} and  $\sigma_{\rm max}(\bm{A}_m)$ is the largest singular value of matrix $\bm{A}_m$.


%
%

Hence, we effectively have the same model as in \eqref{eq:sepNMF}.
Applying Lemma~\ref{lem:gillis}, we see that if
\begin{align*}
\epsilon  &\le  \frac{1}{\sqrt{K}\kappa(\bm{A}_m)} \text{min}\left(\frac{1}{2\sqrt{K-1}},\frac{1}{4}\right)\left(1+80\kappa^2(\bm{A}_m)\right)^{-1} - \frac{2(1+\sqrt{{\rm ln}(1/\delta)})}{\sigma_{\rm max}(\bm{A}_m)\sqrt{S}\eta} \\
	   &= \mathcal{O}\left(\text{max}\left(K^{-1}\kappa^{-3}(\bm{A}_m),\sqrt{{\rm ln}(1/\delta)}(\sigma_{\rm max}(\bm{A}_m)\sqrt{S}\eta)^{-1}\right)\right),
\end{align*}
then, with probability at least $1-\delta$, the SPA algorithm identifies the matrix $\bm{A}_m$ with an error bound given by
\begin{align*}
&\min_{\bm \Pi}\|\widehat{\bm A}_m\bm \Pi- {\bm A}_m\|_{2,\infty} \\
&= \max_{1 \leq j \leq K} \min_{\widehat{q}_k \in \widehat{\varLambda}_q} \left\|\bm{A}_m(:,j) - \overline{\bm{Z}}_m(:,\widehat{q}_k)\right\|_2 \\
& \leq \left(\sigma_{\rm max}(\bm{A}_m)\sqrt{K}\epsilon + \frac{2\sqrt{K}(1+\sqrt{{\rm ln}(1/\delta)})}{\sqrt{S}\eta}\right)\left(1+80\kappa^2(\bm{A}_m)\right) \\
&= \mathcal{O}\bigg(\sqrt{K}\kappa^2(\bm{A}_m) \text{max}\left(\sigma_{\rm max}(\bm{A}_m)\epsilon,\sqrt{{\rm ln}(1/\delta)}(\sqrt{S}\eta)^{-1}\right)\bigg).
\end{align*}
This completes the proof.

\section{Proof of Theorem~\ref{thm:Lm}}\label{app:thm2}

Assuming that the rows of $\overline{\bm{H}}_m$ are generated from the probability simplex uniformly at random, we now analyze under what conditions vectors close to all $K$ vertices of the probability simplex appear in the rows of $\overline{\bm H}_m$.

 Let us denote the probability simplex as $\mathcal{X} = \{ \bm x \in \mathbb{R}^K | \bm{x}^\top \bm{1} = 1 , \bm x \ge 0\}$. 
 
 Let us consider an $\epsilon$-neighbourhood of the $k$-th vertex ${\bm e}_k$  denoted as $\mathcal{Q}_{k}(\epsilon)$ such that
 \begin{align}
     \mathcal{Q}_{k}(\epsilon) := \{{\bm q} \in \mathcal{X} |~ \|{\bm q} - {\bm e}_k \|_2 \le \epsilon \}
 \end{align}
 
Geometrically, the continuous set $\mathcal{Q}_{k}(\epsilon)$ can be considered as the intersection of the probability simplex $\mathcal{X}$ and the euclidean ball of radius $\epsilon$ centered at ${\bm e}_k$, i.e, $\mathcal{Q}_{k}(\epsilon) = \mathcal{X}\cap \mathcal{B}({\bm e}_k,\epsilon)$ (see Fig. \ref{fig:simpedge}).

\begin{figure}[h]
	\centering
	\includegraphics[width=0.4\linewidth]{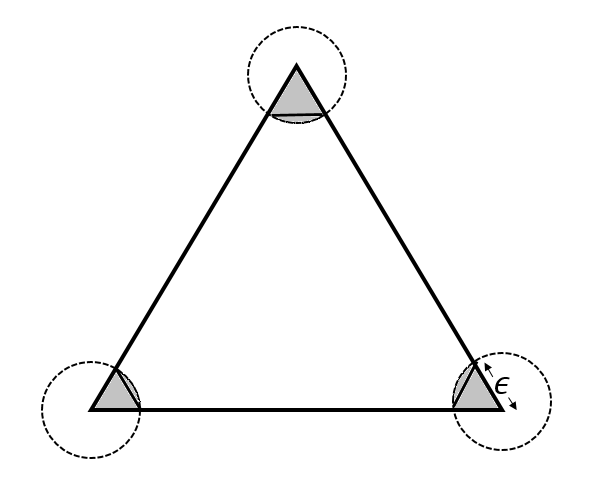}
	\caption{The big triangle represents the probability simplex $\mathcal{X}$ when $K=3$, the dotted circles denotes the euclidean balls $\mathcal{B}({\bm e}_k,\epsilon)$, the shaded region denotes $\mathcal{X}\cap \mathcal{B}({\bm e}_k,\epsilon)$. The small triangles near the vertices has the same volume as the simplex having edge lengths $\epsilon$ denoted as $\mathcal{X}_{\epsilon }$  }
	\label{fig:simpedge}
\end{figure}

Suppose we are uniformly sampling a set $\mathcal{P}$ of size $s$  from the probability simplex $\mathcal{X}$ such that $\mathcal{P} := \{{\bm p}_1,{\bm p}_2,\dots, {\bm p}_s\} $

Let us define an event $J_i$ such that for every $i \in \{1,\dots,s\}$,
\begin{equation}
 J_i=
    \begin{cases}
      1, & \text{if~}\ \bm p_i \in \mathcal{Q}_{k}(\epsilon) \\
      0, & \text{otherwise}
    \end{cases}
\end{equation}
Consider the probability such that event $J_i$ happens,
\begin{align}
    {\sf Pr}(J_i=1) &= \frac{\text{vol}(\mathcal{Q}_{k}(\epsilon))}{\text{vol}(\mathcal{X})}\\
    &= \frac{\text{vol}(\mathcal{X}\cap \mathcal{B}({\bm e}_k,\epsilon))}{\text{vol}(\mathcal{X})}\\
&\ge \frac{\text{vol}(\mathcal{X}_{\epsilon})}{\text{vol}(\mathcal{X})} \label{eqvol_2}\\
    &\ge \left(\frac{\epsilon}{\sqrt{2}}\right)^{K-1} \label{eqvol_1}
\end{align}
where $\mathcal{X}_{\epsilon}$ denotes the $(K-1)$-dimensional simplex which intersects the co-ordinate axes at $\frac{\epsilon}{\sqrt{2}} {\bm e}_k$, for every $k\in \{1,\dots K\}$ and its volume is given by $\frac{(\epsilon/\sqrt{2})^{K-1}}{(K-1)!}$ \cite{stein1966anote}. (Note that the probability simplex $\mathcal{X}$ intersects the co-ordinate axes at ${\bm e}_k$ for every $k$). The inequality in Eq.\eqref{eqvol_2} uses the geometric property that the volume of $\mathcal{X}\cap \mathcal{B}({\bm e}_k,\epsilon)$ is greater than the volume of $\mathcal{X}_{\epsilon}$ (see Fig. \ref{fig:simpedge}). 

Let us define the random variable $U = \sum_{i=1}^s J_i$. 
Then, 
\begin{align}
\mathbb{E}[U] &=  \mathbb{E}[\sum_{i=1}^s J_i] = \sum_{i=1}^s\mathbb{E}[ J_i]\\
&= \sum_{i=1}^s {\sf Pr}(J_i=1) = s {\sf Pr}(J_i=1)\\
&\ge s\left(\frac{\epsilon}{\sqrt{2}}\right)^{K-1} \label{eqexp1}
\end{align}

Now, if there exists at least one sample from set $\mathcal{P}$ which is in the $\epsilon$-neighbourhood of $k$-th vertex, ie the event $J_i$ happens at least once, then $U = \sum_{i=1}^s J_i\ge 1$. We are interested in finding the below probability,
\begin{align}
    {\sf Pr}(U \ge 1) &= 1- {\sf Pr}(U < 1)\\
    &= 1- {\sf Pr}(U \le 0)\\
    &= 1- {\sf Pr}(U = 0)
\end{align}

So, our goal boils down to finding ${\sf Pr}(U \le 0)$ and we will achieve this using Chernoff-Hoeffding bound.
\begin{Lemma} \label{lem:chernoff} \cite{stephane2004concentration}
 Let $J_1,\dots,J_s$ be independent bounded random variables such
that $J_i$ falls in the interval $[a_i,b_i]$ with probability one and let $U=\sum_{i=1}^s J_i$. Then for any $t > 0$,
\begin{align}
    {\sf Pr}(U - \mathbb{E}[U] \le -t) \le e^{-2t^2/\sum_{i=1}^s (b_i-a_i)^2}
\end{align}
\end{Lemma}
It follows that
\begin{align}
    {\sf Pr}(U  \le \mathbb{E}[U]-t) \le e^{-2t^2/\sum_{i=1}^s (b_i-a_i)^2}
\end{align}

By assigning $\mathbb{E}[U]-t=0$, we get $t=\mathbb{E}[U]$ . Also, notice that in our case $b_i=1$, $a_i=0$, then
\begin{align}
    {\sf Pr}(U  \le 0) &\le e^{-\frac{2\mathbb{E}[U]^2}{s}} \\
    &\le e^{-\frac{s\epsilon^{2(K-1)}}{2^{K-2}}} \label{chernoff1}
\end{align}
Eq. \eqref{chernoff1} is obtained by using the inequality \eqref{eqexp1} and implies that, the probability such that the uniform sample $\mathcal{P}$ does not contain any points from $k$-th vertex is less than $e^{-\frac{s\epsilon^{2(K-1)}}{2^{K-2}}}$. 

Now we have to find the corresponding probability that considers all the $K$ vertices.

For this, let us define events $E_k$ as follows,
\begin{align}
    E_k = \{\text{There exists no point }{\bm p}\text{ in the uniform sample set }\mathcal{P}\text{ such that }{\bm p} \in \mathcal{Q}_{k}(\epsilon) \}
\end{align}

From Eq. \eqref{chernoff1}, it is clear that ${\sf Pr}(E_k) \le e^{-\frac{s\epsilon^{2(K-1)}}{2^{K-2}}} $. Since the points are uniformly sampled from the probability simplex $\mathcal{X}$, this bound is applicable for all $k \in \{1,\dots,K\}$.

Now let us define the event $E$ as below
\begin{align*}
    E = \{  \text{there exists at least one point in the uniform sample set }\mathcal{P}\text{ such that }{\bm p} \in \mathcal{Q}_{k}(\epsilon)\text{ for each }k \}
\end{align*}

We can observe that $E =\bigcap_{k=1}^K\overline{E_k} $ where $\overline{E_k}$ is the complement of the event $E_k$.

 Therefore,
\begin{align}
    {\sf Pr}(E) ={\sf Pr}\left(\bigcap_{k=1}^K\overline{E_k}\right) &= {\sf Pr}(\overline{\cup_{k=1}^K E_k}) \nonumber \\
    &= 1- {\sf Pr}\left(\bigcup_{k=1}^K E_k\right) \nonumber \\
    &\ge 1-\sum_{k}{\sf Pr}(E_k)\nonumber \\
    &\ge 1-K e^{-\frac{s\epsilon^{2(K-1)}}{2^{K-2}}} \label{eventprob1}
\end{align}

Eq. \eqref{eventprob1} implies that with probability greater than or equal to $1-K e^{-\frac{s\epsilon^{2(K-1)}}{2^{K-2}}}$, the points from the $\epsilon$-neighbourhood of all the vertices are contained by set $\mathcal{P}$.

If $s$ represents the number of rows in $\overline{\bm H}_m$, then for $s \ge \frac{2^{K-2}}{\epsilon^{2(K-1)}}{\rm log}\big(\frac{K}{\rho}\big)$ , with probability at least $1-\rho$, a uniform sample from the probability simplex $\mathcal{X}$ will contain $\epsilon$-near-vertex points of all the $K$ vertices. Note that $s=(M-1)K$ where $M$ is the number of annotators. This provides a bound on the number of annotators needed. 
Specifically, if there exists at least  $1+  \frac{2^{K-2}}{K\epsilon^{2(K-1)}}{\rm log}\big(\frac{K}{\rho}\big)$ annotators, then we have the conclusion of Theorem~\ref{thm:Lm}

\section{Proof of Theorem~\ref{thm:couple}}\label{app:thm3}
To show this theorem, we will use the following Lemma:
\begin{Lemma}\label{lem:huang}\cite{huang2014non}
	Consider a matrix factorization model $\bm R=\bm P_1\bm P_2^\T$, where $\bm R\in\mathbb{R}^{M\times N}$, $\bm P_1\in\mathbb{R}^{M\times K}$, $\bm P_2\in\mathbb{R}^{N\times K}$, and ${\rm rank}(\bm P_1)={\rm rank}(\bm P_2)=K$. If $\bm P_i\geq \bm 0$ for $i=1,2$ and both $\bm P_1$ and $\bm P_2$ are sufficiently scattered, we have any $\widehat{\bm P}_1\geq \bm 0$ and $\widehat{\bm P}_2\geq\bm 0$ that satisfy $\bm R=\widehat{\bm P}_1\widehat{\bm P}_2^\T$ must have the following form
	\begin{equation}
	\widehat{\bm P}_1 =\bm P_1\bm \Pi\bm \Sigma,~	\widehat{\bm P}_2 =\bm P_2\bm \Pi\bm \Sigma^{-1},
	\end{equation}
	where $\bm \Pi$ is a permutation matrix and $\bm \Sigma^{-1}$ is a diagonal nonnegative singular matrix.
\end{Lemma}
Lemma~\ref{lem:huang} addresses the identifiability of a conventional nonnegative matrix factorization (NMF) model.
Simply speaking, if both latent factors of $\bm R=\bm P_1\bm P_2^\T$ are sufficiently scattered, then the NMF of $\bm R$ is unique up to column permutation and scaling of the latent factors.

Now we start proving Theorem~\ref{thm:couple}.
Let us consider the following matrix $\bm R$:
\begin{equation*}
\bm R=\begin{bmatrix}
\bm R_{m_1,\ell_1}&\bm R_{m_1,\ell_2}&\ldots&\bm R_{m_1,\ell_{|{\cal P}_2|}}\\
\vdots &\vdots&\ldots&\vdots\\
\bm R_{m_{|{\cal P}_1|},\ell_1}&\bm R_{m_1,\ell_2}&\ldots&\bm R_{m_{|{\cal P}_2|},\ell_{|{\cal P}_2|}}\\
\end{bmatrix}
\end{equation*}
It is readily seen that
\begin{align*}
\bm R &= \begin{bmatrix}
\bm A_{m_1}\\
\vdots\\
\bm A_{m_{|{\cal P}_1|}}.
\end{bmatrix}\bm D [\bm A_{\ell_1}^\T,\ldots,\bm A_{\ell_{|{\cal P}_2|}}^\T]\\
&=\bm H^{(1)}\bm D(\bm H^{(2)})^\T.
\end{align*}
It suffices to show that $\bm R=\bm H^{(1)}\bm D(\bm H^{(2)})^\T$ is unique up to column permutations of $\bm H^{(i)}$ for $i=1,2$.
The reason is that if such uniqueness holds, then $\bm A_m$ for $m\notin{\cal P}_1\cup {\cal P}_2$ can be identified with via solving
\[  \bm R_{m,r} =\bm A_m \bm D\bm A_r^\T,~r\in{\cal P}_1\cup {\cal P}_2, \]
up to the same column permutation.

Note that since $\bm D$ and $\bm A_r$ have been identified from $\bm R$, solving $\bm A_m$ amounts to solving a system of linear equations, which has a unique solution under ${\rm rank}(\bm D)={\rm rank}(\bm A_m)=K$.
Under the assumption that all $m\notin {\cal P}_1\cup {\cal P}_2$ are connected to a certain $r\in {\cal P}_1\cup {\cal P}_2$ via ${\bm R}_{m,r}$, all the $\bm A_m$'s for $m\notin{\cal P}_1\cup {\cal P}_2$ can be identified up to the same column permutation.

To show the identifiability of $\bm R=\bm H^{(1)}\bm D(\bm H^{(2)})^\T$, consider re-writing the right hand side as
\begin{align*}
\bm R &=\underbrace{\bm H^{(1)}\bm D^{1/2}}_{\bm P_1} \underbrace{ \bm D^{1/2}(\bm H^{(2)})^\T}_{\bm P_2^\T}  \\
& = \bm P_1\bm P_2^\T.
\end{align*}     
It suffices to show that $\bm P_1$ and $\bm P_2$ are unique up to column permutation and scaling, since the constraints on $\bm 1^\T\bm A_m=\bm 1^\T(\Rightarrow \bm 1^\T\bm H^{(i)} = |{\cal P}_i|)$ removes the scaling ambiguity.

Assume that there is an alternative solution $\bm R =\widehat{\bm P}_1\widehat{\bm P}_2^\T$. By the fact ${\rm rank}(\bm A_m)=K$, we have 
\[  {\rm rank}(\bm H^{(i)})=K \Rightarrow  {\rm rank}(\bm P_i)=K. \]
Note that $\bm H^{(1)}$ and $\bm H^{(2)}$ are both sufficiently scattered. This directly implies that both $\bm P_1$ and $\bm P_2$ are sufficiently scattered,
since column scaling of $\bm H^{(1)}$ and $\bm H^{(2)}$ does not affect the cone of their respective transposes, i.e.,
\[   \cone{(\bm H^{(i)})^\T} =   \cone{(\bm H^{(i)}\bm \Sigma)^\T}  \]
for any nonnegative, nonsingular and diagonal $\bm \Sigma$.

Then, by Lemma~\ref{lem:huang}, it must hold that
\[   	\widehat{\bm P}_1 =\bm H^{(1)}\bm D^{1/2}\bm \Pi\bm \Sigma,~	\widehat{\bm P}_2 =\bm H^{(2)}\bm D^{1/2}\bm \Pi\bm \Sigma^{-1},   \]
for a certain $\bm \Pi$ and $\bm \Sigma$.
Nevertheless, $\bm \Sigma$ is automatically removed by the constraints $\bm 1^\T\bm A_m=\bm 1^\T$.

\section{Proof of Theorem~\ref{thm:SSM}} \label{app:thm4}

Let $\overline{\bm H}$ represent the row normalized matrix of ${\bm H}$. Geometrically,   $\overline{\bm H}$ can be viewed as the projection of the rows of ${\bm H}$ onto the $(K-1)$-probability simplex. (cf. Fig.~\ref{fig:coneplot}). 

The work \cite{lin2014identifiability} provides a characterization to the spread of the rows of $\overline{\bm H}$ in the probability simplex using a measure called as uniform pixel purity level (named in the context of hyperspectral imaging). The purity level is denoted by $\gamma$ and defined as follows: 
\begin{align}
    \gamma = \text{sup}\{r \le 1 | \mathcal{R}(r) \subseteq \text{conv}\{\overline{\bm H}^{\top} \}               \}
\end{align}
where 
\begin{align}
    \mathcal{R}(r) = \{\x \in \mathbb{R}^K | \|\x\|_2 \le r\} \cap \text{conv}\{{\bm e}_1,\dots,{\bm e}_K\}
\end{align}

Geometrically, $\mathcal{R}(\gamma)$ is the `largest' $\mathcal{R}(r)$ that can be inscribed inside $\text{conv}\{\overline{\bm H}^{\top}\}$. There is an interesting link between $\gamma$ and the sufficiently scattered condition shown in \cite{fu2016robust}:
\begin{Lemma}\label{lem:pure}\cite{fu2016robust}
Assume $K \ge 3$ holds. If $\gamma \ge \frac{1}{\sqrt{K-1}}$, then $\overline{\H}$ is sufficiently scattered.
\end{Lemma}

In general cases, it is hard to check if $\gamma \ge \frac{1}{\sqrt{K-1}}$ holds \cite{huang2014non}.
However, in \cite{lin2014identifiability}, a sufficient condition for $\gamma \ge \frac{1}{\sqrt{K-1}}$ is derived:

\begin{Lemma}\label{lem:cond}\cite{lin2014identifiability}
 Suppose the following assumption holds true: for every $k,k^{'} \in \{1,\dots,K\}, k\neq k^{'}$, there exist a row index $q_{k k^{'}}$ in $\overline{{\bm H}}$ such that
\begin{align}
\overline{{\bm H}}(q_{k k^{'}},:) =  \alpha_{k k^{'}}{\bm e}_k^{\top}  +(1-\alpha_{k k^{'}}){\bm e}_{k^{'}}^{\top} \label{assump}
\end{align}
where $\frac{1}{2} < \alpha_{k k^{'}} < 1$ for $K\ge 4$, $\frac{2}{3} < \alpha_{k k^{'}} < 1$ for $K = 3$ and $\alpha_{k k^{'}} = 1$, for $K=2$ . Then $\gamma \ge \frac{1}{\sqrt{K-1}}$ and by Lemma \ref{lem:pure}, $\overline{\bm H}$ is sufficiently scattered.
\end{Lemma}



\begin{figure}[h]
	\centering
	\includegraphics[width=0.6\linewidth]{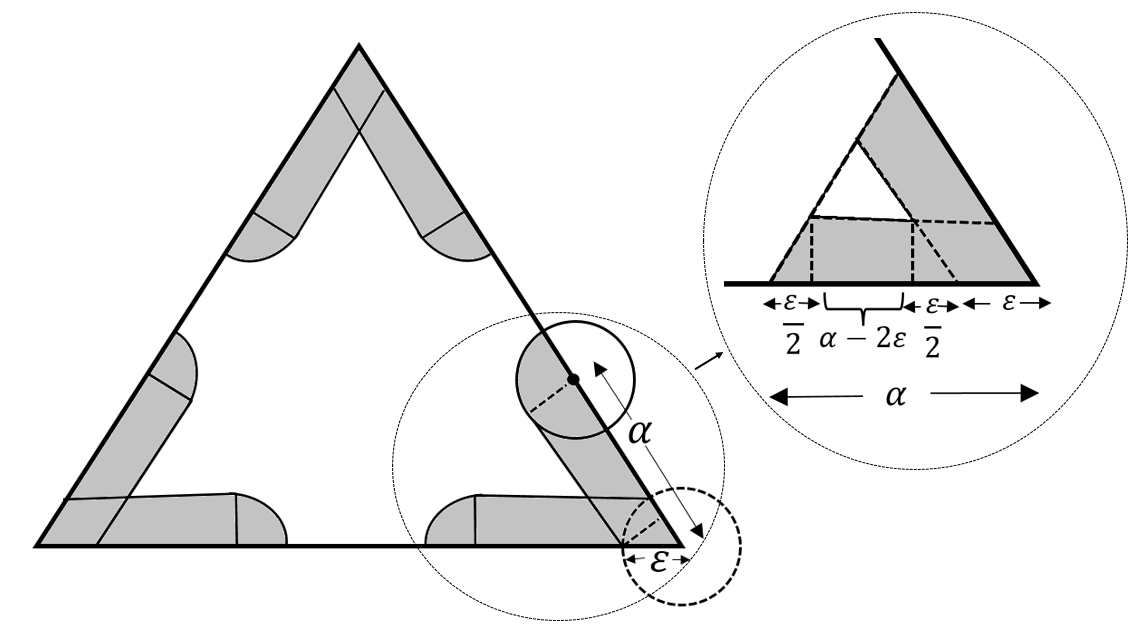}
	\caption{($\varepsilon$-sufficiently scatterd) The big triangle represents the probability simplex $\mathcal{X}$ when $K=3$, the shaded region denotes the region which is $\varepsilon$ near the edges and in the inset, the shaded region depicts the lower bound for the volume (area) of this region at each vertex. }
	\label{fig:suffpedge}
\end{figure}
%

Let us define $\alpha = \underset{{k, k^{'} \in \{1,\dots,K\},k\neq k^{'}}}{\text{min}} \alpha_{k k^{'}}$. Lemma \ref{lem:cond} states that for every edges of the probability simplex, if there exists at least one row  in $\overline{\bm H}$ which belongs to certain range in the edge which is of length $\alpha$ ($\alpha$-edge) (see Fig. \ref{fig:suffpedge}), then the matrix ${\overline{\bm H}}$ is sufficiently scattered. 

Let us denote the probability simplex as $\mathcal{X} = \{ \bm x \in \mathbb{R}^K | \bm{x}^\top \bm{1} = 1 , \bm x \ge 0\}$. 

For each vertex, there exists $K-1$ edges associated to it.  Let us denote an $\varepsilon$-neighbourhood of $\alpha$-edge connecting the vertices $k$ and $k^{'}$  as $\widetilde{\mathcal{Q}}_{k,k^{'}}(\varepsilon,\alpha)$. By the conditions in Lemma \ref{lem:cond} and the definition of $\varepsilon$- sufficiently scattered (cf. Def. \ref{def:suffepsilon}), it can be seen that, for every edges connecting $k$ and $k'$,  if there exists at least one row in $\overline{\H}$ belonging to $\widetilde{\mathcal{Q}}_{k,k^{'}}(\varepsilon,\alpha)$, then $\overline{\H}$ is $\varepsilon$-sufficiently scattered.

For each vertex $k$, the union of $\widetilde{\mathcal{Q}}_{k,k^{'}}(\varepsilon,\alpha)$, $k'=\{1,\dots,K-1\}$ forms a continuous neighbourhood around the vertex $k$ denoted as $\widetilde{\mathcal{Q}}_{k}(\varepsilon,\alpha)$ , i.e, 
\begin{align}
\widetilde{\mathcal{Q}}_{k}(\varepsilon,\alpha) = \bigcup_{k'=1}^{K-1} \widetilde{\mathcal{Q}}_{k,k^{'}}(\varepsilon,\alpha) \label{equnion}
\end{align}

Geometrically, the volume of the continuous set $\widetilde{\mathcal{Q}}_{k}(\varepsilon,\alpha)$ can be lower bounded as below (see Fig \ref{fig:suffpedge})
\begin{equation}
    \text{vol}(\widetilde{\mathcal{Q}}_{k}(\varepsilon,\alpha)) \ge \text{vol}(\mathcal{X}_{\alpha})-\text{vol}(\mathcal{X}_{\alpha-2\varepsilon}) \label{volsuff}
\end{equation}

where $\mathcal{X}_{\alpha'}$ is $(K-1)$-dimensional simplex which intersects the co-ordinate axes at $\frac{\alpha'}{\sqrt{2}} {\bm e}_k$ for every $k=\{1,\dots,K\}$ and thus has the edge lengths $\alpha'$. The volume of $\mathcal{X}_{\alpha'}$ is given by $\frac{(\alpha'/\sqrt{2})^{K-1}}{(K-1)!}$ \cite{stein1966anote}

Eq. \eqref{volsuff} can then be written as
\begin{align}
    \text{vol}(\widetilde{\mathcal{Q}}_{k}(\varepsilon,\alpha)) &\ge \frac{\alpha^{K-1}}{\sqrt{2}^{K-1}(K-1)!}-\frac{(\alpha-2\varepsilon)^{K-1}}{\sqrt{2}^{K-1}(K-1)!}\\
    &=\frac{\alpha^{K-1}}{\sqrt{2}^{K-1}(K-1)!}\left(1-\left(1-\frac{2\varepsilon}{\alpha}\right)^{K-1}\right)\\
    &=\frac{\alpha^{K-1}}{\sqrt{2}^{K-1}(K-1)!}\left(\frac{2\varepsilon}{\alpha}\left(1+\left(1-\frac{2\varepsilon}{\alpha}\right)+\dots+\left(1-\frac{2\varepsilon}{\alpha}\right)^{K-2}\right)\right) \label{eqbinom}\\
    &=\frac{\alpha^{K-1}}{\sqrt{2}^{K-1}(K-1)!}\left(\frac{2\varepsilon}{\alpha}+\frac{2\varepsilon}{\alpha}\left(1-\frac{2\varepsilon}{\alpha}\right)+\dots+\frac{2\varepsilon}{\alpha}\left(1-\frac{2\varepsilon}{\alpha}\right)^{K-2}\right)\\
    &\ge  \frac{\alpha^{K-1}}{\sqrt{2}^{K-1}(K-1)!}\frac{2\varepsilon}{\alpha} \label{ineqbinom}\\
    &= \frac{\alpha^{K-2}\varepsilon}{\sqrt{2}^{K-3}(K-1)!} \label{eqvolumeQ}
\end{align}

Eq. \eqref{eqbinom} uses the geometric series sum formula $1-a^n = (1-a)(1+a+\dots+a^{n-1})$ and the assumption that $\alpha > 2\varepsilon$.

From Eq. \eqref{equnion} and \eqref{eqvolumeQ}, the volume of the set $\widetilde{\mathcal{Q}}_{k,k^{'}}(\varepsilon,\alpha)$ can be lower bounded as
\begin{align}
 \text{vol}\left(\bigcup_{k'=1}^{K-1} \widetilde{\mathcal{Q}}_{k,k^{'}}(\varepsilon,\alpha) \right) &= \text{vol}\left(\widetilde{\mathcal{Q}}_{k}(\varepsilon,\alpha)\right)  \\
 \implies (K-1)\text{vol}\left(\widetilde{\mathcal{Q}}_{k,k'}(\varepsilon,\alpha)\right) &\ge \text{vol}\left(\widetilde{\mathcal{Q}}_{k}(\varepsilon,\alpha)\right)\\
 &\ge \frac{\alpha^{K-2}\varepsilon}{\sqrt{2}^{K-3}(K-1)!}\\
 \implies \text{vol}\left(\widetilde{\mathcal{Q}}_{k,k'}(\varepsilon,\alpha)\right) &\ge \frac{\alpha^{K-2}\varepsilon}{\sqrt{2}^{K-3}(K-1)(K-1)!}
\end{align}
 
Suppose we are uniformly sampling a set $\mathcal{P}$ of size $s$  from the probability simplex $\mathcal{X}$ such that $\mathcal{P} := \{{\bm p}_1,{\bm p}_2,\dots, {\bm p}_s\} $

Let us define an event $J_i$ such that for every $i \in \{1,\dots,s\}$,
\begin{equation}
 J_i=
    \begin{cases}
      1, & \text{if~}\ \bm p_i \in \widetilde{\mathcal{Q}}_{k,k'}(\varepsilon,\alpha) \\
      0, & \text{otherwise}
    \end{cases}
\end{equation}
Consider the probability such that event $J_i$ happens,
\begin{align}
    {\sf Pr}(J_i=1) &= \frac{\text{vol}(\widetilde{\mathcal{Q}}_{k,k'}(\varepsilon,\alpha))}{\text{vol}(\mathcal{X})}\\
    &\ge (K-1)!\frac{\alpha^{K-2}\varepsilon}{\sqrt{2}^{K-3}(K-1)(K-1)!} \label{eqvol2}\\
    &= \frac{\alpha^{K-2}\varepsilon}{(K-1)\sqrt{2}^{K-3}} \label{eqvol1}
\end{align}
  Eq. \eqref{eqvol1} uses the fact that the volume of the $(K-1)$-dimensional simplex $\mathcal{X}$ is given by $\frac{1}{(K-1)!}$ \cite{stein1966anote}.

Now, let us define the random variable $U = \sum_{i=1}^s J_i$. 
Then, 
\begin{align}
\mathbb{E}[U] &=  \mathbb{E}[\sum_{i=1}^s J_i] = \sum_{i=1}^s\mathbb{E}[ J_i]\\
&= \sum_{i=1}^s {\sf Pr}(J_i=1) = s {\sf Pr}(J_i=1)\\
&\ge s\frac{\alpha^{K-2}\varepsilon}{(K-1)\sqrt{2}^{K-3}} \label{eqexp}
\end{align}

Now, if there exists at least one sample from set $\mathcal{P}$ which is in the $\varepsilon$-neighbourhood of $\alpha$-edge, i.e, the event $J_i$ happens at least once, then $U = \sum_{i=1}^s J_i\ge 1$. Also,
\begin{align}
    {\sf Pr}(U \ge 1) &= 1- {\sf Pr}(U < 1)\\
    &= 1- {\sf Pr}(U \le 0)\\
    &= 1- {\sf Pr}(U = 0)\\
\end{align}

So, our goal boils down to finding ${\sf Pr}(U \le 0)$ and we will achieve this using Lemma \ref{lem:chernoff}.  

From Lemma \ref{lem:chernoff}, it follows that
\begin{align}
    {\sf Pr}(U  \le \mathbb{E}[U]-t) \le e^{-2t^2/\sum_{i=1}^s (b_i-a_i)^2}
\end{align}

By assigning $\mathbb{E}[U]-t=0$, we get $t=\mathbb{E}[U]$ . Also, notice that in our case $b_i=1$, $a_i=0$, then
\begin{align}
    {\sf Pr}(U  \le 0) &\le e^{-\frac{2\mathbb{E}[U]^2}{s}} \\
    &\le e^{-\frac{s\alpha^{2(K-2)}\varepsilon^{2}}{2^{K-4}(K-1)^2}} \label{chernoff}
\end{align}
Eq. \eqref{chernoff} is obtained by using the inequality $\mathbb{E}[U]  \ge s\frac{\alpha^{K-2}\varepsilon}{(K-1)\sqrt{2}^{K-3}} $ as in Eq. \eqref{eqexp} and implies that, the probability such that the uniform sample $\mathcal{P}$ does not contain any points from $\widetilde{\mathcal{Q}}_{k,k'}(\varepsilon,\alpha)$ is less than $e^{-\frac{s\alpha^{2(K-2)}\varepsilon^{2}}{2^{K-4}(K-1)^2}}$. 

Now we have to find the corresponding probability that considers all the $(K-1)$ edges for each vertex $k$.

For this, let us define events $\widetilde{E}_{kk^{'}}$ as follows,
\begin{align}
    \widetilde{E}_{kk^{'}} = \{\text{There exists no point }{\bm p}\text{ in the uniform sample set }\mathcal{P}\text{ such that }{\bm p} \in \mathcal{Q}_{k,k^{'}}(\varepsilon,\alpha) \}
\end{align}

From Eq. \eqref{chernoff}, it is clear that ${\sf Pr}(\widetilde{E}_{kk^{'}}) \le e^{-\frac{s\alpha^{2(K-2)}\varepsilon^{2}}{2^{K-4}(K-1)^2}} $. Since the points are uniformly sampled from the probability simplex $\mathcal{X}$, this bound is applicable for all $k,k^{'} \in \{1,\dots,K\}, k\neq k^{'}$.

Now let us define the event $\widetilde{E}$ as below
\begin{align*}
    \widetilde{E} = \{ \text{there exists at least one point in the set }\mathcal{P}\text{ such that }{\bm p} \in \mathcal{Q}_{k,k^{'}}(\varepsilon,\alpha)\text{ for all } k,k^{'}, k\neq k^{'} \}
\end{align*}

We can observe that $\widetilde{E} ={\bigcap}_{\substack{k,k^{'} \vspace{-0.4em}\\
  k \neq k^{'} }}\overline{\widetilde{E}_{kk^{'}}} $ where $\overline{\widetilde{E}_{kk^{'}}}$ is the complement of the event $\widetilde{E}_{kk^{'}}$.

 Therefore,
\begin{align}
    {\sf Pr}(E) ={\sf Pr}\left({\bigcap}_{\substack{k,k^{'} \vspace{-0.4em}\\
  k \neq k^{'} }}\overline{\widetilde{E}_{kk^{'}}}\right) &= {\sf Pr}(\overline{{\bigcup}_{\substack{k,k^{'} \vspace{-0.4em}\\
  k \neq k^{'} }} \widetilde{E}_{kk^{'}}}) \nonumber \\
    &= 1- {\sf Pr}\left({\bigcup}_{\substack{k,k^{'} \vspace{-0.4em}\\
  k \neq k^{'} }} \widetilde{E}_{kk^{'}}\right) \nonumber \\
    &\ge 1-\sum_{k,k^{'}}{\sf Pr}(\widetilde{E}_{kk^{'}})\nonumber \\
    &\ge 1-K(K-1) e^{-\frac{s\alpha^{2(K-2)}\varepsilon^{2}}{2^{K-4}(K-1)^2}} \label{eventprob}
\end{align}

Eq. \eqref{eventprob} implies that with probability greater than or equal to $1-K(K-1) e^{-\frac{s\alpha^{2(K-2)}\varepsilon^{2}}{2^{K-4}(K-1)^2}}$, the points from the $\varepsilon$-neighbourhood of all the $\alpha$-edges are contained by set $\mathcal{P}$. This essentially means that the rows of $\overline{\bm H}$ satisfy the assumption \eqref{assump} with $\varepsilon$ accuracy and thus the $\varepsilon$-sufficiently scattered condition is achieved.

 From Lemma \ref{lem:cond}, we get the lower bounds for $\alpha$ under various values of $K$:
 \begin{align}
     \alpha > \alpha_{min} , \quad  \alpha_{min} = 
  \begin{cases}
    1, & \text{for } K=2, \\
    \frac{2}{3}, & \text{for } K=3, \\
    \frac{1}{2}, & \text{for } K > 3.
  \end{cases}
 \end{align}
 
 Therefore, Eq. \eqref{eventprob} can be agian bounded as,
 \begin{align}
     {\sf Pr}(E) \ge 1-K(K-1) e^{-\frac{s\alpha_{min}^{2(K-2)}\varepsilon^{2}}{2^{K-4}(K-1)^2}}
 \end{align}

If $s$ represents the number of rows in $\overline{\bm H}$, then for $s \ge \frac{2^{K-4}(K-1)^2}{\alpha_{min}^{2(K-2)}\varepsilon^{2}}{\rm log}\big(\frac{K(K-1)}{\rho}\big)$ , with probability at least $1-\rho$, ${\bm H}$ is $\varepsilon$-sufficienty scattered. Note that $s=(M-1)K$ where $M$ is the number of annotators. This provides a bound on the number of annotators needed. 

Consequently, if there exists at least  $1+  \frac{2^{K-4}(K-1)^2}{K\alpha_{min}^{2(K-2)}\varepsilon^{2}}{\rm log}\big(\frac{K(K-1)}{\rho}\big)$ annotators, then we have the conclusion of Theorem~\ref{thm:SSM}.

\section{Sample complexity for second order and third order statistics} \label{sample}
In this section, we compare the sample complexity needed to estimate the second order statistics of the annotator responses from $m$ and $\ell$ denoted as ${\bm R}_{m,\ell}$ and the third order statistics of the annotator responses from $m$, $n$ and $\ell$ denoted as ${\bm R}_{m,n,\ell}$ given a dataset of $N$ samples to jointly label as one of the $K$ classes.

In crowdsourcing, not all samples are labeled by an annotator. To be specific, an annotator $m$ labels each sample with probability $p_m \in (0,1]$ and in most of the practical cases, $p_m << 1$. For simpler analysis, let us take $p_m = p$, for all annotators. Then, this results to have an average of $\ceil{Np^2}$ joint responses from annotators $m$ and ${\ell}$ and $\ceil{Np^3}$ joint responses from annotators $m$, $n$ and $\ell$. With this and using the matrix and tensor concentration results from  \cite{zhang2014spectral}, the estimation error for ${\bm R}_{m,\ell}$ and ${\bm R}_{m,n,\ell}$ can be re-stated as, with probability at least $1-\delta$,
\begin{align}
    \|\bm R_{m,\ell}-\widehat{ \bm R}_{m,\ell}\|_{F}\leq \frac{1+\sqrt{{\rm log}(1/\delta)}}{p\sqrt{N}} \label{sample1}\\
     \| \bm R_{m,n,\ell}-\widehat{ \bm R}_{m,n,\ell}\|_{F}\leq \frac{1+\sqrt{{\rm log}(K/\delta)}}{p^{\frac{3}{2}}\sqrt{N/K}}  \label{sample2} 
\end{align}

It is clear from Eq. \eqref{sample1} and \eqref{sample2} that in order to achieve the same accuracy, third order statistics need much higher number of samples compared to second order statistics when $p$ is smaller and $K$ is larger. 





\begin{thebibliography}{41}
\providecommand{\natexlab}[1]{#1}
\providecommand{\url}[1]{\texttt{#1}}
\expandafter\ifx\csname urlstyle\endcsname\relax
  \providecommand{\doi}[1]{doi: #1}\else
  \providecommand{\doi}{doi: \begingroup \urlstyle{rm}\Url}\fi

\bibitem[Anandkumar et~al.(2014)Anandkumar, Ge, Hsu, and
  Kakade]{anandkumar2014tensor}
Anandkumar, A., Ge, R., Hsu, D., and Kakade, S.~M.
\newblock A tensor approach to learning mixed membership community models.
\newblock \emph{The Journal of Machine Learning Research}, 15\penalty0
  (1):\penalty0 2239--2312, 2014.

\bibitem[Arora et~al.(2013)Arora, Ge, Halpern, Mimno, Moitra, Sontag, Wu, and
  Zhu]{arora2012practical}
Arora, S., Ge, R., Halpern, Y., Mimno, D., Moitra, A., Sontag, D., Wu, Y., and
  Zhu, M.
\newblock A practical algorithm for topic modeling with provable guarantees.
\newblock In \emph{Proceedings of ICML}, 2013.

\bibitem[Baraniuk et~al.(2008)Baraniuk, Davenport, DeVore, and
  Wakin]{baraniuk2008simple}
Baraniuk, R., Davenport, M., DeVore, R., and Wakin, M.
\newblock A simple proof of the restricted isometry property for random
  matrices.
\newblock \emph{Constructive Approximation}, 28\penalty0 (3):\penalty0
  253--263, 2008.

\bibitem[Bertsekas(1999)]{bertsekas1999nonlinear}
Bertsekas, D.~P.
\newblock \emph{Nonlinear programming}.
\newblock Athena Scientific, 1999.

\bibitem[Chan et~al.(2011)Chan, Ma, Ambikapathi, and Chi]{VMAX}
Chan, T.-H., Ma, W.-K., Ambikapathi, A., and Chi, C.-Y.
\newblock A simplex volume maximization framework for hyperspectral endmember
  extraction.
\newblock \emph{IEEE Trans. Geosci. Remote Sens.}, 49\penalty0 (11):\penalty0
  4177 --4193, Nov. 2011.

\bibitem[Dalvi et~al.(2013)Dalvi, Dasgupta, Kumar, and
  Rastogi]{dalvi2013aggregating}
Dalvi, N., Dasgupta, A., Kumar, R., and Rastogi, V.
\newblock Aggregating crowdsourced binary ratings.
\newblock In \emph{Proceedings of the 22Nd International Conference on World
  Wide Web}, pp.\  285--294, New York, NY, USA, 2013. ACM.

\bibitem[Dawid \& Skene(1979)Dawid and Skene]{dawid1979maximum}
Dawid, A.~P. and Skene, A.~M.
\newblock Maximum likelihood estimation of observer error-rates using the em
  algorithm.
\newblock \emph{Applied statistics}, pp.\  20--28, 1979.

\bibitem[{Deng} et~al.(2009){Deng}, {Dong}, {Socher}, {Li}, and
  and]{deng2009imagenet}
{Deng}, J., {Dong}, W., {Socher}, R., {Li}, L., and and.
\newblock Imagenet: A large-scale hierarchical image database.
\newblock In \emph{2009 IEEE Conference on Computer Vision and Pattern
  Recognition}, pp.\  248--255, June 2009.

\bibitem[Dietterich(2000)]{dietterich2000ensemble}
Dietterich, T.~G.
\newblock Ensemble methods in machine learning.
\newblock In \emph{International workshop on multiple classifier systems}, pp.\
   1--15. Springer, 2000.

\bibitem[Donoho \& Stodden(2003)Donoho and Stodden]{donoho2003does}
Donoho, D. and Stodden, V.
\newblock When does non-negative matrix factorization give a correct
  decomposition into parts?
\newblock In \emph{Advances in neural information processing systems},
  volume~16, 2003.

\bibitem[Fu et~al.(2015)Fu, Ma, Chan, and Bioucas-Dias]{fu2014self}
Fu, X., Ma, W.-K., Chan, T.-H., and Bioucas-Dias, J.~M.
\newblock Self-dictionary sparse regression for hyperspectral unmixing: Greedy
  pursuit and pure pixel search are related.
\newblock \emph{IEEE J. Sel. Topics Signal Process.}, 9\penalty0 (6):\penalty0
  1128--1141, 2015.

\bibitem[Fu et~al.(2016)Fu, Huang, Yang, Ma, and Sidiropoulos]{fu2016robust}
Fu, X., Huang, K., Yang, B., Ma, W.-K., and Sidiropoulos, N.~D.
\newblock Robust volume minimization-based matrix factorization for remote
  sensing and document clustering.
\newblock \emph{IEEE Trans. Signal Process.}, 64\penalty0 (23):\penalty0
  6254--6268, 2016.

\bibitem[Fu et~al.(2018{\natexlab{a}})Fu, Huang, and
  Sidiropoulos]{fu2018identifiability}
Fu, X., Huang, K., and Sidiropoulos, N.~D.
\newblock On identifiability of nonnegative matrix factorization.
\newblock \emph{IEEE Signal Process. Lett.}, 25\penalty0 (3):\penalty0
  328--332, 2018{\natexlab{a}}.

\bibitem[Fu et~al.(2018{\natexlab{b}})Fu, Huang, Sidiropoulos, and
  Ma]{fu2018nonnegative}
Fu, X., Huang, K., Sidiropoulos, N.~D., and Ma, W.-K.
\newblock Nonnegative matrix factorization for signal and data analytics:
  Identifiability, algorithms, and applications.
\newblock \emph{arXiv preprint arXiv:1803.01257}, 2018{\natexlab{b}}.

\bibitem[Ghosh et~al.(2011)Ghosh, Kale, and McAfee]{ghosh2011moderates}
Ghosh, A., Kale, S., and McAfee, P.
\newblock Who moderates the moderators?: crowdsourcing abuse detection in
  user-generated content.
\newblock In \emph{Proceedings of the 12th ACM conference on Electronic
  commerce}, pp.\  167--176. ACM, 2011.

\bibitem[Gillis(2014)]{gillis2014and}
Gillis, N.
\newblock The why and how of nonnegative matrix factorization.
\newblock \emph{Regularization, Optimization, Kernels, and Support Vector
  Machines}, 12:\penalty0 257, 2014.

\bibitem[Gillis \& Vavasis(2014)Gillis and Vavasis]{Gillis2012}
Gillis, N. and Vavasis, S.
\newblock Fast and robust recursive algorithms for separable nonnegative matrix
  factorization.
\newblock \emph{IEEE Trans. Pattern Anal. Mach. Intell.}, 36\penalty0
  (4):\penalty0 698--714, April 2014.

\bibitem[Huang et~al.(2014)Huang, Sidiropoulos, and Swami]{huang2014non}
Huang, K., Sidiropoulos, N., and Swami, A.
\newblock Non-negative matrix factorization revisited: Uniqueness and algorithm
  for symmetric decomposition.
\newblock \emph{IEEE Trans. Signal Process.}, 62\penalty0 (1):\penalty0
  211--224, 2014.

\bibitem[Huang et~al.(2016)Huang, Sidiropoulos, and Liavas]{huang2016flexible}
Huang, K., Sidiropoulos, N.~D., and Liavas, A.~P.
\newblock A flexible and efficient algorithmic framework for constrained matrix
  and tensor factorization.
\newblock \emph{IEEE Trans. Signal Process.}, 64\penalty0 (19):\penalty0
  5052--5065, 2016.

\bibitem[Huang et~al.(2018)Huang, Fu, and Sidiropoulos]{huang2018learning}
Huang, K., Fu, X., and Sidiropoulos, N.~D.
\newblock Learning hidden markov models from pairwise co-occurrences with
  applications to topic modeling.
\newblock In \emph{Proceedings of ICML 2018}, 2018.

\bibitem[Jonker \& Volgenant(1986)Jonker and Volgenant]{jonker1986improving}
Jonker, R. and Volgenant, T.
\newblock Improving the hungarian assignment algorithm.
\newblock \emph{Operations Research Letters}, 5\penalty0 (4):\penalty0
  171--175, 1986.

\bibitem[Karger et~al.(2011)Karger, Oh, and Shah]{karger2011budget}
Karger, D.~R., Oh, S., and Shah, D.
\newblock Budget-optimal crowdsourcing using low-rank matrix approximations.
\newblock 2011.

\bibitem[Karger et~al.(2013)Karger, Oh, and Shah]{karger2013efficient}
Karger, D.~R., Oh, S., and Shah, D.
\newblock Efficient crowdsourcing for multi-class labeling.
\newblock \emph{ACM SIGMETRICS Performance Evaluation Review}, 41\penalty0
  (1):\penalty0 81--92, 2013.

\bibitem[Karger et~al.(2014)Karger, Oh, and Shah]{karger2014budget}
Karger, D.~R., Oh, S., and Shah, D.
\newblock Budget-optimal task allocation for reliable crowdsourcing systems.
\newblock \emph{Operations Research}, 62\penalty0 (1):\penalty0 1--24, 2014.

\bibitem[Kittur et~al.(2008)Kittur, Chi, and Suh]{kittur2008crowdsourcing}
Kittur, A., Chi, E.~H., and Suh, B.
\newblock Crowdsourcing user studies with mechanical turk.
\newblock In \emph{Proceedings of the SIGCHI conference on human factors in
  computing systems}, pp.\  453--456. ACM, 2008.

\bibitem[Kolda \& Bader(2009)Kolda and Bader]{kolda2009tensor}
Kolda, T.~G. and Bader, B.~W.
\newblock Tensor decompositions and applications.
\newblock \emph{SIAM review}, 51\penalty0 (3):\penalty0 455--500, 2009.

\bibitem[Lease \& Kazai.(2011)Lease and Kazai.]{lease2011Overview}
Lease, M. and Kazai., G.
\newblock Overview of the trec 2011 crowdsourcing track.
\newblock 2011.

\bibitem[Lin et~al.(2015)Lin, Ma, Li, Chi, and
  Ambikapathi]{lin2014identifiability}
Lin, C.-H., Ma, W.-K., Li, W.-C., Chi, C.-Y., and Ambikapathi, A.
\newblock Identifiability of the simplex volume minimization criterion for
  blind hyperspectral unmixing: The no-pure-pixel case.
\newblock \emph{IEEE Trans. Geosci. Remote Sens.}, 53\penalty0 (10):\penalty0
  5530--5546, Oct 2015.

\bibitem[Liu et~al.(2012)Liu, Peng, and Ihler]{liu2012variational}
Liu, Q., Peng, J., and Ihler, A.~T.
\newblock Variational inference for crowdsourcing.
\newblock In \emph{Advances in neural information processing systems}, pp.\
  692--700, 2012.

\bibitem[Nascimento \& Bioucas-Dias(2005)Nascimento and Bioucas-Dias]{VCA}
Nascimento, J. and Bioucas-Dias, J.
\newblock Vertex component analysis: A fast algorithm to unmix hyperspectral
  data.
\newblock \emph{IEEE Trans. Geosci. Remote Sens.}, 43\penalty0 (4):\penalty0
  898--910, 2005.

\bibitem[Raykar et~al.(2010)Raykar, Yu, Zhao, Valadez, Florin, Bogoni, and
  Moy]{raykar2010learning}
Raykar, V.~C., Yu, S., Zhao, L.~H., Valadez, G.~H., Florin, C., Bogoni, L., and
  Moy, L.
\newblock Learning from crowds.
\newblock \emph{Journal of Machine Learning Research}, 11\penalty0
  (Apr):\penalty0 1297--1322, 2010.

\bibitem[Razaviyayn et~al.(2013)Razaviyayn, Hong, and
  Luo]{razaviyayn2013unified}
Razaviyayn, M., Hong, M., and Luo, Z.-Q.
\newblock A unified convergence analysis of block successive minimization
  methods for nonsmooth optimization.
\newblock \emph{SIAM Journal on Optimization}, 23\penalty0 (2):\penalty0
  1126--1153, 2013.

\bibitem[Robert(2014)]{robert2014machine}
Robert, C.
\newblock Machine learning, a probabilistic perspective, 2014.

\bibitem[Sidiropoulos et~al.(2017)Sidiropoulos, De~Lathauwer, Fu, Huang,
  Papalexakis, and Faloutsos]{sidiropoulos2017tensor}
Sidiropoulos, N.~D., De~Lathauwer, L., Fu, X., Huang, K., Papalexakis, E.~E.,
  and Faloutsos, C.
\newblock Tensor decomposition for signal processing and machine learning.
\newblock \emph{IEEE Trans. Signal Process.}, 65\penalty0 (13):\penalty0
  3551--3582, 2017.

\bibitem[Snow et~al.(2008)Snow, O'Connor, Jurafsky, and Ng]{snow2008cheap}
Snow, R., O'Connor, B., Jurafsky, D., and Ng, A.~Y.
\newblock Cheap and fast---but is it good?: evaluating non-expert annotations
  for natural language tasks.
\newblock In \emph{Proceedings of the conference on empirical methods in
  natural language processing}, pp.\  254--263. Association for Computational
  Linguistics, 2008.

\bibitem[Stein(1966)]{stein1966anote}
Stein, P.
\newblock {A Note on the Volume of a Simplex}.
\newblock \emph{The American Mathematical Monthly}, 73\penalty0 (3), 1966.
\newblock \doi{10.2307/2315353}.

\bibitem[Stephane~Boucheron(2004)]{stephane2004concentration}
Stephane~Boucheron, Gabor~Lugosi, O.~B.
\newblock {Concentration Inequalities}, 2004.
\newblock URL: \url{http://www.econ.upf.edu/~lugosi/mlss_conc.pdf}.

\bibitem[Traganitis et~al.(2018)Traganitis, Pages-Zamora, and
  Giannakis]{traganitis2018blind}
Traganitis, P.~A., Pages-Zamora, A., and Giannakis, G.~B.
\newblock Blind multiclass ensemble classification.
\newblock \emph{IEEE Trans. Signal Process.}, 66\penalty0 (18):\penalty0
  4737--4752, 2018.

\bibitem[Welinder et~al.(2010)Welinder, Branson, Perona, and
  Belongie]{welinder2010multidimensional}
Welinder, P., Branson, S., Perona, P., and Belongie, S.~J.
\newblock The multidimensional wisdom of crowds.
\newblock In \emph{Advances in neural information processing systems}, pp.\
  2424--2432, 2010.

\bibitem[Zhang et~al.(2014)Zhang, Chen, Zhou, and Jordan]{zhang2014spectral}
Zhang, Y., Chen, X., Zhou, D., and Jordan, M.~I.
\newblock Spectral methods meet em: A provably optimal algorithm for
  crowdsourcing.
\newblock In \emph{Advances in neural information processing systems}, pp.\
  1260--1268, 2014.

\bibitem[Zhou et~al.(2014)Zhou, Liu, Platt, and Meek]{zhou2014aggregated}
Zhou, D., Liu, Q., Platt, J., and Meek, C.
\newblock Aggregating ordinal labels from crowds by minimax conditional
  entropy.
\newblock In \emph{Proceedings of ICML}, volume~32, pp.\  262--270, 2014.

\end{thebibliography}
\end{document}